\documentclass[a4paper,12pt]{amsart}
\usepackage[utf8]{inputenc}
\usepackage[T1]{fontenc}
\usepackage[UKenglish]{babel}
\usepackage[a4paper,margin=28mm]{geometry}
\usepackage{verbatim}

\allowdisplaybreaks[4]
\usepackage{times}
\usepackage{dsfont,mathrsfs}
\usepackage{amsmath}
\usepackage{amsthm}
\usepackage{amssymb}
\usepackage{amsfonts}
\usepackage{latexsym}
\usepackage{booktabs}
\usepackage{mathtools}
 \usepackage[ruled,linesnumbered]{algorithm2e}
\usepackage{hyperref}%建立label的超链接
\hypersetup{hypertex=true,
	colorlinks=true,
	linkcolor=blue,
	anchorcolor=blue,
	citecolor=blue,
	pdftitle={Overleaf Example},
	pdfpagemode=FullScreen}
\usepackage{xcolor}

\newtheorem{theorem}{Theorem}[section]
\newtheorem{lemma}[theorem]{Lemma}
\newtheorem{proposition}[theorem]{Proposition}

\newtheorem{assumption}{Assumption}

\theoremstyle{definition}

\newtheorem{remark}[theorem]{Remark}

\numberwithin{equation}{section}

\renewcommand{\labelenumi}{\roman{enumi}}
\renewcommand\theenumi\labelenumi
\renewcommand{\leq}{\leqslant}
\renewcommand{\le}{\leqslant}
\renewcommand{\geq}{\geqslant}
\renewcommand{\ge}{\geqslant}

\newcommand{\R}{\mathbb{R}}
\newcommand{\E}{\mathbb{E}}

\usepackage{marginnote}%this avoids floating marginal notes
\title[Approximation to Deep Q-Network by Stochastic Delay Differential Equations]{Approximation to Deep Q-Network by Stochastic Delay Differential Equations}

\author[J. Lu]{Jianya Lu}
\address[J.~Lu]{School of Mathematics, Statistics and Actuarial Science, University of Essex, UK}
\email{jianya.lu@essex.ac.uk}

\author[Y. Mo]{Yingjun Mo}
\address[Y.~Mo]{1. Department of Mathematics, Faculty of Science and Technology, University of Macau, Macau, 999078, China; 2. Zhuhai UM Science, Technology Research Institute, Zhuhai, 519031, China}
\email{yc27477@um.edu.mo}

\begin{document}
	

\begin{abstract}
Despite the significant breakthroughs that the Deep Q-Network (DQN) has brought to reinforcement learning, its theoretical analysis remains limited. In this paper, we construct a stochastic differential delay equation (SDDE) based on the DQN algorithm and estimate the Wasserstein-1 distance between them. We provide an upper bound for the distance and prove that the distance between the two converges to zero as the step size approaches zero. This result allows us to understand DQN’s two key techniques, the experience replay and the target network, from the perspective of continuous systems. Specifically, the delay term in the equation, corresponding to the target network, contributes to the stability of the system. Our approach leverages a refined Lindeberg principle and an operator comparison to establish these results.
\end{abstract}

\maketitle

\section{Introduction}
Reinforcement Learning (RL)\cite{bertsekas2019reinforcement,sutton2018reinforcement} has been a prominent field of machine learning, and has gained tons of attention in recent decades. It studies sequential decision-making through interactions with environments to develop a policy that determines actions based on the current state to maximise long-term return.

Q-learning is one of the most fundamental learning strategies in RL, designed to make optimal decisions through the action-value function. Since its introduction by \cite{watkins1992q}, it has been extensively studied. However, in large-scale and continuous state spaces, as well as in scenarios where observed data exhibits strong correlations, the algorithm becomes unstable and may no longer be applicable. The deep Q-Network (DQN) introduced in the seminal work by \cite{mnih2015human} achieved a breakthrough. Besides combining Q-learning with deep neural networks, the DQN puts forward two novel and crucial tricks, an experience replay, and a target network. This groundbreaking achievement has spurred further exploration in the realm of deep reinforcement learning, leading to the development of approaches such as Double DQN \cite{van2016deep}, Dueling DQN \cite{wang2016dueling}.  In terms of applications, \cite{xu2012url} first used deep reinforcement learning for autonomic cloud management in computer science, which has a similar idea to DQN. 

Despite the significant success of DQN in practice, a deep theoretical understanding of its underlying mechanism, especially the two tricks, remains limited. In this paper, we construct a stochastic differential delay equation (SDDE) based on the DQN iteration and demonstrate that the weight of the action value function in the iteration of DQN is close to the solution of the SDDE in the Wasserstein-1 distance. This diffusion approximation enables us to analyze DQN from the perspective of continuous systems, providing insights into the role of its two tricks during the learning process.

The trick of experience replay allows historical data to be stored in a replay buffer and randomly sampled during the iteration process, enabling the data to be approximately treated as independent and identically distributed (i.i.d.). This property is crucial for analyzing DQN through a continuous system. The corresponding stochastic differential equation has a unique solution, which also provides a continuous perspective on the convergence of the DQN algorithm.

The trick of the target network is that the weights of the Q-network are updated periodically. In the construction of the SDDE, this technique corresponds to the delay term in the equation. An SDDE incorporates the history of the process and combines it with the current state for prediction. As shown by \cite{scheutzow2013exponential}, the introduction of a delay in the diffusion term of an SDDE often reduces the fluctuations of a stochastic system. In other words, SDDEs allow for more information about the state, leading to more stable dynamics, as evidenced by a smaller variance. In contrast, Q-learning without a target network corresponds to a stochastic differential equation (SDE), which relies solely on the current state for future predictions. The absence of a delay term results in greater instability compared to SDDE-based models. This perspective offers a continuous-system interpretation of the role of the two tricks.

\subsection{Literature review and motivations}

Since Q-learning and the further breakthrough DQN were proposed, the relevant theoretical analysis of deep Q-learning algorithms has attracted attention.  
\cite{fan2020theoretical} focused on the fitted Q-iteration algorithm, which is a simplified version of DQN, with sparse ReLU networks. \cite{MR4725527} studied the global convergence of the Q-learning algorithm with an i.i.d. observation model and action-value function approximation based on a two-layer neural network. % \cite{xu2020finite} studied the non-asymptotic convergence of a neural Q-learning algorithm under non-i.i.d. observations.

The main limitation of the aforementioned work lies in the lack of analysis of the role of the original DQN algorithm, particularly regarding the mechanisms of experience replay and target networks. Some literature analyzed the experience replay mechanism of the DQN algorithm based on specific conditions. For example, \cite{szlak2021convergence} provided a convergence rate guarantee of Q-learning with experience replay in the setting of tabular. \cite{9536409} provides a theoretical analysis of a popular version of deep Q-learning with experience replay under realistic and verifiable assumptions by adopting a dynamical systems perspective. Meanwhile, some literature analyzed the target network mechanism. \cite{carvalho2020new} established the convergence of Q-learning combining target network in DQN with linear function approximation. A theoretical explanation of its two mechanisms simultaneously is still lacking. See \cite{liu2022understanding,zhang2023convergence} for more details.

For stochastic algorithms, it is natural to consider it as a discretization to a continuous dynamic for a given step size. Several studies, see \cite{li2017stochastic}, \cite{li2019stochastic}, have focused on constructing SDEs corresponding to stochastic algorithms, providing crucial insights from the perspective of continuous systems. This diffusion approximation serves as a bridge that enables the application of continuous dynamic analysis methods to investigate the properties of stochastic algorithms.  In particular, \cite{yin2008q} was the first paper which studied the $Q$-learning from the point of view of differential equations and proposed its possible connection with SDEs. Notably, in recent years, \cite{MR4409807} analyzed SVRG and its related SDDE, \cite{guillin2024error} established a quantitative error estimate between stochastic gradient descent with momentum and the underdamped Langevin diffusion. Taking these factors into account, we aim to explore the continuous-time approximation of DQN and understand the mechanisms of this algorithm from the viewpoint of stochastic differential delay dynamics.

Our main result establishes a meaningful connection between DQN and an SDDE, and provides a perspective of stochastic delay systems to understand DQN. More precisely, under some appropriate assumptions that Q-network has certain smoothness properties, we establish an error bound for the approximation. This theorem shows that the approximation error converges to $0$ as the step size $\eta$ approaches $0$. Drawing upon the principles of SDDE theory, specifically when the diffusion term incorporates a delay, dissipation arises to mitigate fluctuations originating from Brownian motion, as exemplified in \cite{scheutzow2013exponential}. This insight intuitively explains why the algorithm of DQN is effective in variance reduction and thus stable in training the associated neural networks.

\subsection{Notations and organization}
The Euclidean norm of $x \in \R^d$ and the inner product of $x, y \in \R^d$ are denoted by $|x|$ and $\langle x, y\rangle$, respectively. For matrix $A\in \R^{d\times d},\|A\|_{\mathrm{HS}}$ is the Hilbert-Schmidt norm.

The paper is organized as follows. In the second section, we construct the mathematical model of the DQN algorithm and expound on the assumptions made in this paper and their justification. In the third section, we derive the expression for the stochastic delay differential equation corresponding to the DQN iterative formula, elucidate the properties inherent in this equation, and finally state the main theorem of this paper. At last, we provide the proof for the aforementioned theorem.

\section{Background and setting}{\label{section2}}

RL can be analysed as a Markov decision process (MDP) with the tuple $(\mathcal{S}, \mathcal{A}, p, r)$, where $\mathcal{S}$ is the state space with the element state $s$, $\mathcal{A}$ is the action space with the element action $a$, $p$ is the transition probability kernel and $r$ is the immediate reward. At time $t=0,1,2,\cdots$, the agent takes action $a_t$ at the current state $s_t$, then the state transitions to $s_{t+1}$ according to the transition probability $p(\cdot|s_t,a_t)$ and receives reward $r(s_t,a_t):=r_t$. The reward $r_t$ is bounded and $\E[r(s,a)]=R(s,a)$ for any action $a$ and state $s$. The goal of RL is to find a policy $\pi$ to maximise the long-term reward $\sum_{t=0}^\infty \gamma^t r_t$ where $\gamma\in(0,1)$ is a discount parameter. Let the expectation of the long-term reward following policy $\pi$ at $s$ and $a$ be the the action-value function 
$Q(s,a)=\E^\pi[\sum_{t=0}^\infty \gamma^tr_t|s,a]$.

In practice, the observed reward feedback is often noisy in practice (e.g., when rewards are collected through sensors), making it less credible. Moreover, in applications like robotics, a deep reinforcement learning algorithm can be susceptible to manipulation, producing arbitrary errors when exposed to corrupted rewards, see \cite{wang2020reinforcement} for more details. We assume that $r(s,a)$ satisfies a normal distribution $\mathcal{N} (R(s,a), V^2(s,a))$ with variance $V^2(s,a)$, similar assumptions can be found in  \cite{biyik2023active,miyamoto2021convergence}. For the action-value function $Q(s,a)$, as $\gamma\in(0,1)$, we can see that $Q(s,a)$ is bounded by its definition.  Therefore, we introduce the following hypothesis:
\begin{assumption}\label{assum:1}
	The reward $r(s,a)\sim \mathcal{N}(R(s,a),V^2(s,a))$, where $\mathcal{N}(R(s,a),V^2(s,a))$ is a normal distribution with expectation $R(s,a)$ and standard deviation $V(s,a)$, and $R, V:\mathcal{S} \times \mathcal{A} \rightarrow \mathbb{R}$ is bounded continuous.
\end{assumption}
The optimal strategy $\pi^*$ can be obtained through the corresponding optimal action-value function $Q^*$. To find $Q^*$, \cite{watkins1992q} proposed the Q-learning algorithm, whose iteration is given by:
$$ Q(s, a) \leftarrow Q(s, a)+\eta \cdot\left[r(s, a)+\gamma \cdot \max _{a^{\prime}\in \mathcal{A}} Q\left(s^{\prime}, a^{\prime}\right)-Q(s, a)\right],$$
%Here, $r(s, a)$ is the immediate reward of taking action $a\in \mathcal{A}$ in state $s\in \mathcal{S}$, the state space $\mathcal{S}$ is a finite set of states, and $\mathcal{A}$ is a finite set of actions. $\eta$ is the learning rate, $\gamma$ is the discount factor, and 
where $\eta > 0$ is the step size satisfying the standard assumptions of nonsummability, $s^{\prime}$ is the next state that the agent takes action $a$ at state $s$. 

For better application to extensive state 
action spaces, it is natural to assume that the action-value function is a neural network that depends on the parameter $\theta$, the action-value function can be then approximated by calculating the parameter $\theta$. The corresponding algorithm is revised as:
$$\theta\leftarrow \theta+\eta \cdot\left[r(s, a)+\gamma \cdot \max _{a^{\prime}\in \mathcal{A}} Q\left(s^{\prime}, a^{\prime};\theta\right)-Q(s, a;\theta)\right]\nabla_\theta Q(s,a;\theta),$$
where the $Q(\cdot,\cdot;\theta)$ denotes the neural network with parameter $\theta$. 

On the basis of Q-learning and deep neural network, \cite{mnih2015human} introduced the DQN algorithm which employed two tricks, namely experience replay and target network. This greatly improved the effect of the algorithm and achieved breakthrough results, the DQN algorithm is shown as Algorithm \ref{algorithm} in Appendix \ref{appdix:algorithm}. 

Let us give a brief explanation of this algorithm. The experience replay is to record transitions the data $\left(s_t, a_t, r_t, s_{t}'\right)$, with $s_t^{'}$ being the state at the time $t+1$, at each time $t$ in the experience replay memory $\mathcal{M}$ and use the elements in $\mathcal{M}$ to train neural networks. This strategy significantly enhances the accuracy of gradient estimation for stochastic optimization problems.  
%{\jianya We assume in this paper that the replay buffer $\mathcal{M}$ is sufficiently large at the beginning. Although new data is added to  $\mathcal{M}$ during the iteration process, its distribution changes slowly due to the large buffer size. Given that our approximation is considered over a finite time horizon, the buffer’s distribution does not change much and therefore we assume that the samples drawn from the buffer are i.i.d., see Fan et al. (2020) for a similar assumption.}

The trick of the target network aims to obtain an unbiased estimator for the mean-squared Bellman error used in training the Q-network. It is updated as the following: given an initial $\theta_0$, let $\theta^{-}=\theta_0$, we update the neural network parameter $\theta_t$ with $1 \le t \le m$ by the following mini-batch SGD: {
	\begin{equation} \label{e:loopm}
		\ \ \ \ \ \  \ \theta_{t}=\theta_{t-1}+\eta \cdot \frac{1}{|{H}|} \sum_{i \in H_t}\left[r_i + \gamma\max_{a\in\mathcal{A}} Q(s_i',a;\theta^{-})
		-Q\left(s_i, a_i;\theta_{t-1}\right)\right] \nabla_\theta Q\left(s_i, a_i;\theta_{t-1}\right)
	\end{equation}
	where the minibatch $\left\{\left(s_i, a_i, r_i, s_i^{\prime}\right)\right\}_{i \in H_t}$, with a length $H$,  is randomly
	drawn from $\mathcal M$.} Note that the above SGD is designed to minimize
$$\ell\left(\theta\right)=\mathbb{E}_{\left(s, a, r, s^{\prime}\right) \sim U(\mathcal M)}\left[\left(r(s,a)+\gamma \max _{a^{\prime}\in \mathcal{A}} Q\left(s^{\prime}, a^{\prime} ; \theta^{-}\right)-Q\left(s, a ; \theta\right)\right)^2\right],$$
where $U(\mathcal{M})$ is the uniform distribution on $\mathcal{M}$. {After $m$ iterations, we update $\theta^{-}$ by $\theta_m$, i.e. $\theta^{-} \leftarrow \theta_m$,  continue the iteration of the next $m$-length internal loop. 
	At the $k$-th internal loop, we do $\theta^{-} \leftarrow \theta_{(k-1)m}$ and run the internal iteration as \eqref{e:loopm} with $(k-1)m+1 \le t \le k m$.} The DQN algorithm can be generally represented as 
\begin{equation} \label{e:DQN-General}
	\theta_{t}=\theta_{t-1}+\eta \cdot \frac{1}{|{H}|} \sum_{i \in H_t}\left[r_i + \gamma\max_{a\in\mathcal{A}} Q(s_i',a;\theta_{\lfloor \frac{t-1}m\rfloor m})
	-Q\left(s_i, a_i;\theta_{t}\right)\right] \nabla_\theta Q\left(s_i, a_i;\theta_{{t-1}}\right) 
\end{equation}
for $t \ge 1$, where $x \in \mathbb{R}$, $\lfloor x \rfloor$ is the largest integer less than or equal to $x$.

In the DQN algorithm, although new data is added to the replay buffer $\mathcal{M}$ during the iteration, its distribution changes very slowly due to the large buffer size. Given that our approximation is considered over a finite time horizon, the buffer’s distribution remains relatively stable. Therefore, we assume that the samples drawn from the buffer are i.i.d., see \cite{fan2020theoretical} for a similar assumption. At the same time, to capture the exploration of the algorithm and the distributional changes induced by the exploration, we add noise into the iteration, which can be interpreted as an exploration of the parameter space. See \cite{MR4279772,ishfaq2021randomized} for similar ideas. Notably, this noise can also help the algorithm escape from saddle points or local minima, see \cite{MR4095347}.

Based on the above considerations, we have the following iteration,
\begin{equation} \label{e:DQN-General-N}
	\begin{split}
		\theta_{t}=\theta_{t-1}+\eta \cdot \frac{1}{|{ H}|} \sum_{i \in H_t}\left[r_i + \gamma\max_{a\in\mathcal{A}} Q(s_i',a;\theta_{\lfloor \frac{t-1}m\rfloor m})
		-Q\left(s_i, a_i;\theta_{t}\right)\right] & \nabla_\theta Q\left(s_i, a_i;\theta_{{t-1}}\right) \\
		&+\sqrt{\eta\delta}W_t, 
	\end{split}
\end{equation}
for $t \ge 1$, where $W_t$ are i.i.d. {random variables} with the distribution $N(0, I_{d})$ with $I_d$ being the $d \times d$ identity matrix, $\delta > 0$ is the inverse temperature parameter. 

At the end of the section, we introduce the assumption for the neural network. 
\begin{assumption}\label{assum:2}
	The deep neural network of action-value $Q(s,a;\theta)$ satisfies:
	\begin{itemize}
		\item[(a)]  {$\sup _{s \in \mathcal{S}, a \in \mathcal{A}, \theta \in \mathbb R^d}|Q(s, a ; \theta)| \leq {C}$, for some ${C}>0$.}
		\item[(b)] $Q(s, a ; \theta)$ is bounded continuous differentiable from the first to the fourth order with respect to the $\theta$-coordinate for every $s \in \mathcal{S}$ and $a \in \mathcal{A}$.
		\item [(c)] {The activation functions in $Q(s, a; \theta)$ are continuous.}
	\end{itemize}
\end{assumption}

\begin{remark}{\label{remarkQ}}
	(i) As the activation functions are sigmoid functions, then the condition (b) obviously holds. \cite{fan2020theoretical} also assumes the same condition for the analysis of the fitted Q-iteration algorithm. (ii) Let us fix arbitrary $\hat{a} \in \mathcal{A}$ and $\hat{\theta} \in \mathbb{R}^d$. From Assumption (c), we know $Q(\cdot, \hat{a}; \hat{\theta})$ is composed (via addition and multiplication) of continuous functions (activation units), we get that $Q$ is continuous in the $s$-coordinate. (iii)
	$Q$, $\nabla_\theta Q$ are continuous in the $a$-coordinate, since $\mathcal{A}$ is compact metrizable as it is a finite.
\end{remark}

\section{Main Result}{\label{Result}}
In this section we will construct the SDDE based on the algorithm \eqref{e:DQN-General-N} and give our main result, which is the distance between the output of the algorithm and the SDDE solution. These stochastic dynamics offer much needed insight to the algorithms under considerations.  For the convenience of analysis, without loss of generality, we will consider the case of $H=1$ from now on.

\iffalse
In order to approximate Deep Q-network algorithm by stochastic differential delay equations (SDDEs), we have made a relatively mathematical description of the above algorithm. 

Denote by $k(n,i)$ can be regard as randomly chosen from $[H]$ and independent of each other, and the reward $r(s,a)\sim \mathcal{N}(R(s,a),V(s,a))$. The iteration of DQN weights in the algorithm can be represented as

{As the the randomly selected transitions $(s_i, a_i)$ from large reply buffer can be approximately regarded as i.i.d. distributed data, it is natural to assume that $(s_i,a_i)$  are independently sampled from the distribution $q(s, a)$. }
\fi

According to the Assumption \ref{assum:1} that the reward $r(s,a)$ satisfies normal distribution, we can rewrite \eqref{e:DQN-General-N} as follows,  
\begin{eqnarray}{\label{original}}
	\theta_{n+1}&=&\theta_{n}+\eta \left[R(s_n,a_n) + \gamma\max_{a\in\mathcal{A}} Q(s_n',a;\theta_{\lfloor \frac{n}{m} \rfloor m})
	-Q\left(s_n, a_n;\theta_n\right)\right] \cdot \nabla_{\theta} Q\left(s_n, a_n;\theta_n\right)\nonumber\\
	&&+(\eta\beta_n(\theta_n)+\sqrt{\eta\delta}I_d)W_{n+1}\\
	&:=&\theta_n-\eta b_n(\theta_n,\theta_{\lfloor \frac{n}{m} \rfloor m})+(\eta\beta_n(\theta_n)+\sqrt{\eta\delta}I_d)W_{n+1},\nonumber
\end{eqnarray}
where
\begin{eqnarray*}
	\beta_n(\theta_{n}) &=&\mathrm{diag}\left(V(s_n, a_n)\nabla_{\theta}Q(s_n,a_n; \theta_{n})\right)\\
	b_n(\theta_n,\theta_{\lfloor \frac{n}{m} \rfloor m})&=&-\left(R(s_n,a_n) + \gamma\max_{a\in\mathcal{A}} Q(s_n',a;\theta_{\lfloor \frac{n}{m} \rfloor m})
	-Q\left(s_n, a_n;\theta_n\right)\right) \cdot \nabla_{\theta} Q\left(s_n, a_n;\theta_n\right).
\end{eqnarray*}
We can further rearrange the equation above and get 
\begin{eqnarray}\label{e:dqn3}  
	\theta_{n+1}
	&=&\theta_n-\eta \E\left[b_n(\theta_n,\theta_{\lfloor \frac{n}{m} \rfloor m})|\theta_n,\theta_{\lfloor \frac{n}{m}\rfloor m}\right]\nonumber\\
	&&+\eta \E\left[b_n(\theta_n,\theta_{\lfloor \frac{n}{m} \rfloor m})|\theta_n,\theta_{\lfloor \frac{n}{m}\rfloor m}\right]-\eta b_n(\theta_n,\theta_{\lfloor \frac{n}{m} \rfloor m})+(\eta\beta_n(\theta_n)+\sqrt{\eta\delta}I_d)W_{n+1}\nonumber\\
	&:=&\theta_n-\eta \E\left[b_n(\theta_n,\theta_{\lfloor \frac{n}{m} \rfloor m})|\theta_n,\theta_{\lfloor \frac{n}{m}\rfloor m}\right]
	+\sqrt\eta\sigma_n(\theta_n,\theta_{\lfloor \frac{n}{m}\rfloor m}, W_{n+1}).
\end{eqnarray}
For the second term of right hand side of \eqref{e:dqn3}, {thanks to the experience replay trick, it is natural to assume that $(s_n, a_n, s_n^{'})$ are i.i.d. (recall $s_n^{'}$ is $s_{n+1}$) , let us denote the distribution of  $(s_n,a_n)$ by $q(s,a)$ and the transition probability of $s_n'$ by $p(s_n'|s_n,a_n)$, then we have}
\begin{eqnarray*}
	&&\E\left[b_n(\theta_n,\theta_{\lfloor \frac{n}{m} \rfloor m})|\theta_n,\theta_{\lfloor \frac{n}{m}\rfloor m}\right]\\
	&=&-\mathbb E_{(s, a)\sim q}{\left[\left(R(s, a)+\gamma \cdot \bar Q\left(s, a; \theta_{\lfloor \frac{n}{m} \rfloor\cdot m}\right)-Q\left(s, a; \theta_n\right)\right)\nabla Q_{\theta_n}\left(s, a; \theta_{n}\right)\right]}\\
	&:=& b\left(\theta_n, \theta_{\lfloor \frac{n}{m} \rfloor\cdot m}\right),
\end{eqnarray*}
where
\begin{eqnarray*}
	\bar Q\left(s, a; \theta_{\lfloor \frac{n}{m} \rfloor\cdot m}\right):=\int \max _{a^{\prime} \in \mathcal{A}} Q\left(s^{\prime}, a^{\prime} ; \theta_{\lfloor \frac{n}{m} \rfloor\cdot m}\right) p\left(\mathrm{d} s^{\prime} \mid s, a\right).
\end{eqnarray*}
For the term  $\sigma_n(\theta_n,\theta_{\lfloor \frac{n}{m}\rfloor m}, W_{n+1})$ of \eqref{e:dqn3}, it is easy to verify that
\begin{eqnarray*}
	\E\left[\sigma_n(\theta_n,\theta_{\lfloor \frac{n}{m}\rfloor m}, W_{n+1})|\theta_n,\theta_{\lfloor \frac{n}{m}\rfloor m}\right]&=&0,\\
	\mathrm{Cov}\left[\sigma_n(\theta_n,\theta_{\lfloor \frac{n}{m}\rfloor m}, W_{n+1})|\theta_n,\theta_{\lfloor \frac{n}{m}\rfloor m}\right]&=&\eta\Sigma(\theta_{n}, \theta_{\lfloor \frac{n}{m} \rfloor\cdot m})+\eta\bar \beta(\theta_{n})+\delta I_d,
\end{eqnarray*}
where 
\begin{align*}
	&\Sigma(\theta_{n}, \theta_{\lfloor \frac{n}{m} \rfloor\cdot m}) :=\mathbb E[b_n(\theta_{n}, \theta_{\lfloor \frac{n}{m} \rfloor\cdot m})-b(\theta_{n}, \theta_{\lfloor \frac{n}{m} \rfloor\cdot m})][b_n(\theta_{n}, \theta_{\lfloor \frac{n}{m} \rfloor\cdot m})-b(\theta_{n}, \theta_{\lfloor \frac{n}{m} \rfloor\cdot m})]^T\\
	&\qquad\qquad\qquad=\mathbb E\left[b_n(\theta_{n-1}, \theta_{\lfloor \frac{n}{m} \rfloor\cdot m})b_n(\theta_{n-1}, \theta_{\lfloor \frac{n}{m} \rfloor\cdot m})^T\right]-\left[b(\theta_{n-1}, \theta_{\lfloor \frac{n}{m} \rfloor\cdot m})b(\theta_{n-1}, \theta_{\lfloor \frac{n}{m} \rfloor\cdot m})^T\right],\\
	&\bar \beta(\theta_{n}) :=\mathbb E_{(s, a)\sim q}[V(s, a)\nabla_{\theta} Q\left(s,a ; \theta_{n}\right)][V(s,a)\nabla_{\theta} Q\left(s,a; \theta_{n}\right)]^T.
\end{align*}

Combining the analysis above, we can rewrite the DQN algorithm \eqref{e:DQN-General-N} as
\begin{eqnarray}\label{e:dqn4}  
	\theta_{n+1}
	&=&\theta_n-\eta b(\theta_n,\theta_{\lfloor \frac{n}{m} \rfloor m})
	+\sqrt\eta\sigma_n(\theta_n,\theta_{\lfloor \frac{n}{m}\rfloor m}, W_{n+1}),\quad n\ge0.
\end{eqnarray}
According to analysis of term $\sigma_n$, we naturally consider the SDDE
\begin{align}{\label{SDDE}}
	\mathrm{d}{X}_t=-b\left(X_t,X_{\lfloor \frac{t}{m\eta} \rfloor\cdot m\eta}\right)\mathrm{d}t+\sqrt{\eta}\sigma(X_t,X_{\lfloor \frac{t}{m\eta} \rfloor\cdot m\eta})\mathrm{d}B_t,\quad t \geq 0,
\end{align}
where $B_t$ is a standard $d$-dimensional Brownian motion and 
$$
\begin{aligned}
	\sigma\left(x, y\right)&:=\left[\Sigma(x, y)+\bar \beta(x)+\frac{\delta}{\eta}I_d\right]^{1/2}\quad \text{for any}~x,y\in\R^d.
\end{aligned}
$$

To simplify the notation, we denote 
\begin{align}{\label{external}}
	\tilde{\theta}_s=\theta_{m s}\quad\text{and}\quad\tilde{X}_s=X_{s m \eta},
\end{align}
for $s=0,1,2,\cdots$.

Under assumptions, there exists a unique solution to the SDDE (\ref{SDDE}) under Assumption \ref{assum:1} and \ref{assum:2}. From now on, we simply write a number $C_{A_1, \cdots, A_5}$, depending on $A_1, \cdots, A_5$, by $C_A$ in shorthand.

Recall that $W_1$ distance between two probability measures $\mu_1$ and $\mu_2$ is defined as
%$$
%W_1\left(\mu_1, \mu_2\right)=\inf _{(X, Y) \in \mathcal{C}\left(\mu_1, \mu_2\right)} \mathbb{E}|X-Y|,
%$$
%where $\mathcal{C}\left(\mu_1, \mu_2\right)$ is the set of all the coupling realizations of $\mu_1, \mu_2$. By a duality,
$$
W_1\left(\mu_1, \mu_2\right)=\sup _{h \in \operatorname{Lip}(1)}\left|\mu_1(h)-\mu_2(h)\right|,
$$
where $\operatorname{Lip}(1)=\left\{h: \mathbb{R}^d \rightarrow \mathbb{R} ;|h(y)-h(x)| \leq|y-x|\right\}$ and
$
\mu_i(h)=\int_{\mathbb{R}} h(x) \mu_i(\mathrm{d} x)$, $i=1,2$.

The main result of this paper is the following theorem, which provides an approximation error between the distributions of $\tilde{\theta}_s$ and $\tilde{X}_s$.
\begin{theorem}{\label{Th}}
	Assume that the Assumptions \ref{assum:1} and \ref{assum:2} hold. Choosing $0<\delta \leq 1$ and $\eta\leq \min \left\{\delta,\frac{1}{64 L},\frac{L}{8 K^2}\right\} $. Then, for any $T\in \mathbb{N}$, $T>m$, there exists a constant $C_{T,m,A,K, L, d,\beta_{max}, |b(0,0)|}$ such that
	$$
	W_1\left(\mathcal{L}\left({X}_{T\eta}\right), \mathcal{L}\left({\theta}_T\right)\right) 
	\leq C_{T,m,A,K, L,d,\beta_{max}, |b(0,0)|}{(\eta \delta)^{\frac{1}{2}}}\left(1+|\ln \eta|+\frac{\delta}{\eta^{\frac{1}{4}}}\right)\left(\mathbb{E}|{\theta}_0|^4+1\right)^{\frac{7}{4}}.
	$$
\end{theorem}
\begin{remark}
	{Under the assumptions of a Q-network with certain smoothness properties, this theorem provides an error bound for the approximation, elucidating that the approximation error converges to $0$ as the step size $\eta$ approaches $0$.}
\end{remark}

\section{Primary lemmas and the proof of main Theorem }{\label{proof}}
We use two steps to prove Theorem \ref{Th}, the main method is the refined Lindeberg principle \cite{chen2022approximation,chen2023probability}. The first step is to prove an approximation error bound for the internal Markov chains $\{\theta_k\}_{m s \leq k \leq m(s+1)}$ and $\{X_t\}_{m s \eta \leq t \leq m(s+1) \eta}$ in Subsection \ref{sub1}, whereas the second step is to approximate the external Markov chain $\{\tilde{\theta}_s\}_{s \geq 0}$ by $\{\tilde{X}_{s \eta}\}_{s \geq 0}$ in Subsection \ref{sub2}.

Before giving the proof of the main theorem, we first analyze the properties of the parameters of SDDE, i.e., $b\left(x, y\right)$ and $\sigma\left(x, y\right)$, as show in Lemma \ref{property1} and \ref{property2}, which will be proved in Appendix \ref{AAA}. 

\begin{lemma}{\label{property1}}
	Under Assumption \ref{assum:1} and \ref{assum:2}, we have following properties of $b\left(x, y\right)$ and $\sigma\left(x, y\right)$, that is, (i) $b\left(x,y\right)$ is Lipschitz continuous, i.e., there exists a constant $L$, such that
	\begin{equation}{\label{Lipschitz}}
		|b\left(x_1,y_1\right)-b\left(x_2,y_2\right)|\leq L(|x_1-x_2|+|y_1-y_2|).
	\end{equation}
	(ii) There exists a constant $K$, such that
	\begin{align}{\label{growth}}
		\left\|\sigma\left(x, y\right)\right\|_{\mathrm{HS}} \leq K\left|x-y\right|+(K+\sqrt{{\beta_{max}}}+\sqrt{\frac{\delta d}{\eta}}),
	\end{align}
	where $\beta_{max}=\max_{\theta}\|\bar \beta(\theta)\|_{\mathrm{HS}}$.
\end{lemma}
\begin{lemma}{\label{property2}}
	Under Assumption \ref{assum:1} and \ref{assum:2}. There exist constants $A_i \geq 0$ with $i=1,2, \cdots, 5$, such that for any $x, y  \in \mathbb{R}^d$ and unit vectors $v_i \in \mathbb{R}^d$, i.e., $\left|v_i\right|=1, i=1,2,3, b\left(x,y\right)$ satisfies
	\begin{equation}{\label{property2.1}}
		\left|\nabla_{1,v_2} \nabla_{1,v_1} b\left(x,y\right)\right| \leq A_1, \quad\left|\nabla_{1,v_3} \nabla_{1,v_2} \nabla_{1,v_1} b\left(x,y\right)\right| \leq A_2, 
	\end{equation}
	where $\nabla_{1,v}$ denotes the directional derivative of the first coordinate along the direction $v$; 
	and that any $x, y \in \mathbb{R}^d, \sigma$ satisfies
	\begin{align}{\label{property2.2}}
		\left\|\nabla_{1, v_1} \sigma(x, y)\right\|_{\mathrm{HS}}^2 \leq A_3, \quad &\left\|\nabla_{2, v_1} \sigma(x, y)\right\|_{\mathrm{HS}}^2 \leq A_3,\\
		\left\|\nabla_{1, v_2} \nabla_{1, v_1} \sigma(x, y)\right\|_{\mathrm{HS}}^2 \leq A_4,\quad  &\left\|\nabla_{1, v_3} \nabla_{1, v_2} \nabla_{1, v_1} \sigma(x, y)\right\|_{\mathrm{HS}}^2 \leq A_5, \nonumber
	\end{align}
	where $\nabla_{2,v}$ denotes the directional derivative of the second coordinate along the direction $v$.
\end{lemma} 

\begin{lemma}{\label{Markov}}
	(i) Both $(\tilde{\theta}_s)_{s \in \mathbb{Z}^{+}}$ and $(\tilde{X}_s)_{s \in \mathbb{Z}^{+}}$ are Markov chains; 
	(ii) An internal iteration of DQN $\left\{\theta_k\right\}_{0 \leq k \leq m}$ and the solution $\left(X_t\right)_{t \in[0, m \eta]}$ of SDDE (\ref{SDDE}) are time homogeneous Markov chains with states on $\mathbb{R}^d$.
\end{lemma}

We denote $X_{s, t}^x$ with $s \leq t \in[0, \eta]$ to stress the dependence of process on the value $X_s=x$. For the simplicity of notations, we denote $X_{s, t}^x$ by $X_{t-s}^x$ according to time homogeneous property. $\theta_k^x$ is denoted by same way.
\subsection{Approximation of internal Markov chain.}{\label{sub1}}
Let $W \sim \mathcal{N}\left(0, I_d\right)$, which is independent of $I$. The infinitesimal generators of $\left\{\theta_k\right\}_{0 \leq k \leq m}$ and $(X_t)_{t \in[0, m \eta]}$ are respectively
\begin{align}{\label{generate1}}
	\mathcal{A}_j^\theta f(x) &=\mathbb{E}\left[f\left(\theta_{j+1}\right) \mid \theta_j=x\right]-f(x)\\
	&=\mathbb{E}\left[f\left(x-\eta\left[b_n\left(x, \theta_0\right)\right]+(\eta{\beta_I}(x)+\sqrt{\eta\delta}I_d)W\right)\right]-f(x)\nonumber
\end{align}
for $j=0,1,2, \cdots, m-1$, and
\begin{align}{\label{generate2}}
	\mathcal{A}_t^X f(x) & =\lim _{\Delta t \rightarrow 0+} \frac{\mathbb{E}\left[f\left(X_{t+\Delta t}\right) \mid X_t=x\right]-f(x)}{\Delta t} \\
	& =\frac{1}{2} \eta\left\langle \sigma(x,\theta_0)^2, \nabla^2 f(x)\right\rangle_{\mathrm{HS}}-\langle b\left(x,\theta_0\right), \nabla f(x)\rangle \nonumber\\
	& =\frac{1}{2}\left\langle\eta\Sigma(x, \theta_0)+\eta\bar \beta(x) +{\delta}I_d, \nabla^2 f(x)\right\rangle_{\mathrm{HS}}-\langle b\left(x,\theta_0\right), \nabla f(x)\rangle\nonumber
\end{align}
for $t \in[0, m \eta)$. The generators of these two processes do not depend on the time due to time homogeneous property, we shall simply write
\begin{equation}
	\mathcal{A}^X=\mathcal{A}_t^X, \quad \mathcal{A}^\theta=\mathcal{A}_j^\theta .
\end{equation}

Since the diffusion coefficient of SDDE (\ref{SDDE}) is positive definite, by Lemma \ref{property1} (i) and \ref{property2}, we have the following estimates, which will be proved in Appendix \ref{B}.
\begin{lemma}{\label{lem3}}
	Let $X_t$ be the solution to the SDDE (\ref{SDDE}) and denote $P_t h(x)=\mathbb{E}\left[h\left(X_t^x\right)\right]$ for $h \in \operatorname{Lip}(1)$. Then, for any $x \in \mathbb{R}^d$ and unit vectors $v, v_1, v_2, v_3 \in \mathbb{R}^d$, as $\eta \in(0, \delta]$ and $t \in(0,m\eta]$, we have
	\begin{align}{\label{P1}}
		\left|\nabla_{v_1} \left(P_t h\right)(x)\right| \leq e^{m(L+4)},
	\end{align}
	\begin{align}{\label{P2}}
		\left|\nabla_{v_2} \nabla_{v_1}\left(P_t h\right)(x)\right| \leq C_{A,m, L, d} \frac{1}{\sqrt{\delta t}},
	\end{align}
	and
	\begin{align}{\label{P3}}
		\left|\nabla_{v_3} \nabla_{v_2} \nabla_{v_1} P_t h(x)\right| \leq C_{A,m, L, d}\left(1+\frac{1}{\delta t}+\frac{1}{t^{\frac{5}{4}}}\right).
	\end{align}
\end{lemma}
Now, by Lemma \ref{property1}, we can give some moment estimates of SDDE and DQN in Lemma \ref{lem3.1}, \ref{lem3.2}.
\begin{lemma}{\label{lem3.1}}
	Let $X_t$ be the solution to the equation (\ref{SDDE}), $t\leq m\eta$ and $\eta<\frac{L}{8 K^2}$. Then, we have
	\begin{align}{\label{lem3.1.1}}
		\mathbb{E}\left|X_t^x\right|^2 \leq C_{K, L, m,d,\beta_{max}, |b(0,0)|}\left(1+|x|^2+\mathbb{E}\left|\theta_0\right|^2+\delta\right).
	\end{align}
	and
	\begin{align}{\label{lem3.1.2}}
		\mathbb{E}\left|X_t^x-x\right|^2 \leq C_{K, L,m, d,\beta_{max}, |b(0,0)|}\left(1+|x|^2+\mathbb{E}\left|\theta_0\right|^2+\delta\right) t(t+\eta+\delta).
	\end{align}
\end{lemma}
\begin{lemma}{\label{lem3.2}}
	Let $\theta_n^x$ be defined in (\ref{original}), $\delta\leq 1$ and $\eta \leq \min \left\{1,\frac{1}{64 L}\right\}$. Then, for any $0 \leq n \leq m$, we have
	\begin{align}{\label{lem3.2.1}}
		\mathbb{E}\left|\theta_n^x\right|^4 \leq C_{L,m, d,\beta_{max}, |b(0,0)|}\left(1+|x|^4+\mathbb{E}\left|\theta_0\right|^4\right).
	\end{align}
\end{lemma}

Moreover, we can use Lemma \ref{lem3} and Lemma \ref{lem3.1} to prove following lemma.
\begin{lemma}{\label{lem4}}
	Let $Z_t=X_{\eta t}, \mathcal{A}^Z$ be the infinitesimal generator. Let $\mathcal{A}^\theta$ be defined by (\ref{generate1}) and $u_t(x)=\mathbb{E} h\left(X_t^x\right)$ for $0 \leq k \leq m$. Then, as $\eta \leq \min \left\{\delta, \frac{L}{8 K^2}\right\}$, $\delta\leq 1$ and $t \in(0,m\eta]$, we have
	\begin{align*}
		& \left|\mathbb{E} \int_0^1\left[\mathcal{A}^Z u_t\left(Z_s^x\right)-\mathcal{A}^\theta u_t(x)\right] \mathrm{d} s\right| \\
		\leq & C_{A,K, L,m, d,\beta_{max}, |b(0,0)|}\left(1+\frac{1}{t}+\frac{\delta}{t^{\frac{5}{4}}}\right)\left(1+\mathbb{E}\left|\theta_0\right|^4\right)\left(1+|x|^3\right) \eta^{\frac{3}{2}} \delta^{\frac{1}{2}}.
	\end{align*}
\end{lemma}

\begin{proposition}{\label{prop}}
	Assume that the Assumptions \ref{assum:1} and \ref{assum:2} hold. Choosing $\delta\leq 1$, and $\eta\leq \min \left\{\delta,\frac{1}{64 L},\frac{L}{8 K^2}\right\} $, for any $0 \leq k \leq m$, we have
	$$
	W_1\left(\mathcal{L}\left(X_{k \eta}\right), \mathcal{L}\left(\theta_k\right)\right) 
	\leq  C_{A,K, L,m, d,\beta_{max}, |b(0,0)|}\left(1+\mathbb{E}\left|\theta_0\right|^4\right)^{\frac{7}{4}}(\eta \delta)^{\frac{1}{2}}\left(1+|\ln \eta|+\frac{\delta}{\eta^{\frac{1}{4}}}\right).
	$$
\end{proposition}
\begin{proof}
	When $k=0,1$, the result holds obviously. When $k \geq 2$, let $X_0=Y_0=\theta_0$, denote $u_t(x)=\mathbb{E}\left[h\left(X_t^x\right)\right]$, $Z_t=X_{\eta t}$ for $0 \leq l \leq k$ and $h \in \operatorname{Lip}(1)$. For ease of notation, for any $z \in \mathbb{R}^d$, and any $r, t \in \mathbb{Z}^{+}$with $t \geq r$, we denote by $Z_t(t, z)$ the random variable $Z_t$ given $Z_r=z$, and $\theta_t(r, z)$ is similarly defined, it is easy to see
	\begin{equation}{\label{ZZ}}
		Z_t=Z_t\left(r, Z_r\right), \quad \theta_t=\theta_t\left(r, \theta_r\right) .
	\end{equation}
	Then, we have
	$$
	\mathbb{E} h\left(Z_k\right)=\mathbb{E} h\left(Z_k\left(1, Z_1\right)\right)-\mathbb{E} h\left(Z_k\left(1, \theta_1\right)\right)+\mathbb{E} h\left(Z_k\left(1, \theta_1\right)\right),
	$$
	we know $Z_k\left(1, \theta_1\right)=Z_k\left(2, Z_2\left(1, \theta_1\right)\right)$ by (\ref{ZZ}) again, and thus
	$$
	\mathbb{E} h\left(Z_k\left(1, \theta_1\right)\right)=\mathbb{E} h\left(Z_k\left(2, Z_2\left(1, \theta_1\right)\right)\right)-\mathbb{E} h\left(Z_k\left(2, \theta_2\right)\right)+\mathbb{E} h\left(Z_k\left(2, \theta_2\right)\right) .
	$$
	Continue this process with repeatedly using (\ref{ZZ}), we finally obtain
	$$
	\mathbb{E} h\left(Z_k\right)-\mathbb{E} h\left(\theta_k\right)=\sum_{j=1}^k\left[\mathbb{E} h\left(Z_k\left(j, Z_j\left(j-1, \theta_{j-1}\right)\right)\right)-\mathbb{E} h\left(Z_k\left(j, \theta_j\right)\right)\right] .
	$$
	Because $Z_t$ is a time homogeneous Markov chain, we have
	$$
	u_{\eta(k-j)}(z)=\mathbb{E}\left[h\left(X_{\eta k}\right) \mid X_{\eta j}=z\right]=\mathbb{E}\left[h\left(Z_k\right) \mid Z_j=z\right] .
	$$
	Now, by (\ref{ZZ}) and the relation $Z_1^{\theta_{j-1}} \stackrel{\mathrm{d}}{=} Z_j(j-$ $\left.1, \theta_{j-1}\right)$ and $\theta_1^{\theta_{j-1}} \stackrel{\mathrm{d}}{=} \theta_j\left(j-1, \theta_{j-1}\right)$, we have
	$$
	\begin{aligned}
		& \mathbb{E} h\left(Z_k\left(j, Z_j\left(j-1, \theta_{j-1}\right)\right)\right)-\mathbb{E} h\left(Z_k\left(j, \theta_j\right)\right) \\
		= & \mathbb{E} u_{\eta(k-j)}\left(Z_j\left(j-1, \theta_{j-1}\right)\right)-\mathbb{E} u_{\eta(k-j)}\left(\theta_j\right) \\
		= & \mathbb{E} u_{\eta(k-j)}\left(Z_j\left(j-1, \theta_{j-1}\right)\right)-\mathbb{E} u_{\eta(k-j)}\left(\theta_j\left(j-1, \theta_{j-1}\right)\right) \\
		= & \mathbb{E} u_{\eta(k-j)}\left(Z_1^{\theta_{j-1}}\right)-\mathbb{E} u_{\eta(k-j)}\left(\theta_1^{\theta_{j-1}}\right),
	\end{aligned}
	$$
	Hence, we have
	$$
	\mathbb{E} h\left(Z_k\right)-\mathbb{E} h\left(\theta_k\right)=\sum_{j=1}^k\left[\mathbb{E} u_{\eta(k-j)}\left(Z_1^{\theta_{j-1}}\right)-\mathbb{E} u_{\eta(k-j)}\left(\theta_1^{\theta_{j-1}}\right)\right],
	$$
	which further implies
	\begin{align}{\label{lindeberg}}
		W_1\left(\mathcal{L}\left(Z_k\right), \mathcal{L}\left(\theta_k\right)\right) \leq & \sum_{j=1}^{k-1} \sup _{h \in \operatorname{Lip}(1)}\left|\mathbb{E} u_{\eta(k-j)}\left(Z_1^{\theta_{j-1}}\right)-\mathbb{E} u_{\eta(k-j)}\left(\theta_1^{\theta_{j-1}}\right)\right| \\
		& +\sup _{h \in \operatorname{Lip}(1)}\left|\mathbb{E} h\left(Z_1^{\theta_{k-1}}\right)-\mathbb{E} h\left(\theta_1^{\theta_{k-1}}\right)\right| .\nonumber
	\end{align}
	
	Let us now bound each term on the right hand side. Denote the generator of the process $Z_t$ by $\mathcal{A}^Z$. Then, by Itô's formula and the definition of $\mathcal{A}^\theta$, for any $1 \leq j \leq k-1$, we have
	\begin{align}{\label{AA}}
		& \mathbb{E} u_{\eta(k-j)}\left(Z_1^{\theta_{j-1}}\right)-\mathbb{E} u_{\eta(k-j)}\left(\theta_1^{\theta_{j-1}}\right)\nonumber \\
		= & \mathbb{E}\left[u_{\eta(k-j)}\left(Z_1^{\theta_{j-1}}\right)-u_{\eta(k-j)}\left(\theta_{j-1}\right)\right]-\mathbb{E}\left[u_{\eta(k-j)}\left(\theta_1^{\theta_{j-1}}\right)-u_{\eta(k-j)}\left(\theta_{j-1}\right)\right] \nonumber\\
		= & \mathbb{E} \int_0^1\left[\mathcal{A}^Z u_{\eta(k-j)}\left(Z_s^{\theta_{j-1}}\right)-\mathcal{A}^\theta u_{\eta(k-j)}\left(\theta_{j-1}\right)\right] \mathrm{d} s .
	\end{align}
	Since $(k-j) \in(0,m]$, one can derive from Lemma \ref{lem4}, the Hölder inequality and Lemma \ref{lem3.2} that 
	\begin{align*}
		& \sum_{j=1}^{k-1} \sup _{h \in \operatorname{Lip}(1)}\left|\mathbb{E} u_{\eta(k-j)}\left(Z_1^{\theta_{j-1}}\right)-\mathbb{E} u_{\eta(k-j)}\left(\theta_1^{\theta_{j-1}}\right)\right| \\
		\leq & C_{A,K, L,m, d,\beta_{max}, |b(0,0)|} \sum_{j=1}^{k-1}\left(1+\frac{1}{\eta(k-j)}+\frac{\delta}{[\eta(k-j)]^{\frac{5}{4}}}\right)\left(1+\mathbb{E}\left|\theta_0\right|^4\right)\left(1+\mathbb{E}\left|\theta_{j-1}\right|^3\right) \eta^{\frac{3}{2}} \delta^{\frac{1}{2}} \\
		\leq & C_{A,K, L,m, d,\beta_{max}, |b(0,0)|}\left(1+\mathbb{E}\left|\theta_0\right|^4\right)^{\frac{7}{4}} \sum_{j=1}^{k-1}\left(1+\frac{1}{\eta(k-j)}+\frac{\delta}{[\eta(k-j)]^{\frac{5}{4}}}\right) \eta^{\frac{3}{2}} \delta^{\frac{1}{2}} \\
		\leq & C_{A,K, L,m, d,\beta_{max}, |b(0,0)|}\left(1+\mathbb{E}\left|\theta_0\right|^4\right)^{\frac{7}{4}}(\eta \delta)^{\frac{1}{2}}\left(m+|\ln m|+|\ln \eta|+\frac{\delta}{\eta^{\frac{1}{4}}}\right)\\
		\leq & C_{A,K, L,m, d,\beta_{max}, |b(0,0)|}\left(1+\mathbb{E}\left|\theta_0\right|^4\right)^{\frac{7}{4}}(\eta \delta)^{\frac{1}{2}}\left(1+|\ln \eta|+\frac{\delta}{\eta^{\frac{1}{4}}}\right).
	\end{align*}
\end{proof}

\subsection{Approximation of external Markov chain.}{\label{sub2}}
Let $h:\mathbb{R}^d\to\mathbb{R}$ be Lipschitz, $S=\left\lfloor\frac{T}{m}\right\rfloor$, define
$$
U_h(s, x)=\mathbb{E}\left[h\left(\bar{X}_s^x\right)\right], \quad s=0,1,2, \cdots,S
$$
where $\bar{X}_s^x$ stresses that the initial value of $\bar{X}_s$ is $x$, and $\bar{X}_s=X_{s m \eta+(T-mS)\eta}, s=0,1,2, \cdots,S$.

\begin{proof}[Proof of Theorem \ref{Th}] 
By the refined Lindeberg principle, i.e. the same argument as the proof of (\ref{lindeberg}), we have,
{\begin{align*}
		\left|\mathbb{E} h\left(X_{T\eta}\right)-\mathbb{E} h\left(\theta_T\right)\right|
		\leq &\sum_{i=1}^S\left|\mathbb{E} U_h\left(S-i, {X}_{m\eta}^{{\theta}_{m(i-1)}}\right)-\mathbb{E} U_h\left(S-i, {\theta}_m^{{\theta}_{(i-1)m}}\right)\right|\\
		&+\left|\mathbb{E} h\left({X}_{(T-mS)\eta}^{{\theta}_{Sm}}\right)-\mathbb{E} h\left( {\theta}_{T-mS}^{{\theta}_{Sm}}\right)\right|
\end{align*}}
Since for any $x, y \in \mathbb{R}^d$, by (\ref{P1}), we have
\begin{align*}
	\left|U_h(s, x)-U_h(s, y)\right| & =\left|\mathbb{E} h\left(\bar{X}_s^x\right)-\mathbb{E} h\left(\bar{X}_s^y\right)\right| =\left|\mathbb{E} h\left(X_{s m \eta+(T-mS)\eta}^x\right)-\mathbb{E} h\left({X}_{s m \eta+(T-mS)\eta}^y\right)\right| \\
	& =\left|\mathbb{E} h\left(X_{(T-mS)\eta}^{X_{s m \eta}^x}\right)-\mathbb{E} h\left(X_{(T-mS)\eta}^{X_{s m \eta}^y}\right)\right|\\
	&\leq e^{(L+4)(T-mS)\eta} \sup _{h \in \operatorname{Lip}(1)}\left|\mathbb{E} h\left(X_{s m \eta}^x\right)-\mathbb{E} h\left(X_{sm \eta}^y\right)\right|\\
	& =e^{(L+4)(T-mS)\eta} \sup _{h \in \operatorname{Lip}(1)}\left|\mathbb{E} h\left(X_{m \eta}^{X_{(s-1) m \eta}^x}\right)-\mathbb{E} h\left(X_{m \eta}^{X_{(s-1) m \eta}^y}\right)\right|\\
	&\leq e^{(L+4)(T-m(S-1)))\eta} \sup _{h \in \operatorname{Lip}(1)}\left|\mathbb{E} h\left(X_{(s-1) m \eta}^x\right)-\mathbb{E} h\left(X_{(s-1) m \eta}^y\right)\right|\\
	&\leq e^{(L+4)(T-m(S-s))\eta}|x-y| .
\end{align*}
Then, according to Proposition \ref{prop}, we have
$$
\begin{aligned}
	& \sum_{i=1}^S\left|\mathbb{E} U_h\left(S-i, {X}_{m\eta}^{{\theta}_{m(i-1)}}\right)-\mathbb{E} U_h\left(S-i, {\theta}_m^{{\theta}_{(i-1)m}}\right)\right|\\
	\leq & C_{T,m,K, L,A, d,\beta_{max}, |b(0,0)|}(\eta \delta)^{\frac{1}{2}}\left(1+|\ln \eta|+\frac{\delta}{\eta^{\frac{1}{4}}}\right) \sum_{i=1}^S e^{(S-i)(L+4)m\eta}\left(1+\mathbb{E}\left|{\theta}_{(i-1)m}\right|^4\right)^{\frac{7}{4}} .
\end{aligned}
$$
and, by Proposition \ref{prop},
$$
\left|\mathbb{E} h\left({X}_{(T-mS)\eta}^{{\theta}_{Sm}}\right)-\mathbb{E} h\left( {\theta}_{T-mS}^{{\theta}_{Sm}}\right)\right|\leq  C_{m,K, L,A, d,\beta_{max}, |b(0,0)|}(\eta \delta)^{\frac{1}{2}}\left(1+|\ln \eta|+\frac{\delta}{\eta^{\frac{1}{4}}}\right)\left(1+\mathbb{E}\left|\theta_{Sm}\right|^4\right)^{\frac{7}{4}}.
$$
With the help of the proof of Lemma \ref{lem3.2}, we have
$$
\begin{aligned}
	& \left|\mathbb{E} h\left(X_{T\eta}\right)-\mathbb{E} h\left(\theta_T\right)\right| \\
	\leq & C_{T,m,K, L,A, d,\beta_{max}, |b(0,0)|}(\eta \delta)^{\frac{1}{2}}\left(1+|\ln \eta|+\frac{\delta}{\eta^{\frac{1}{4}}}\right)\left(\mathbb{E}|{\theta}_0|^4+1\right)^{\frac{7}{4}} \sum_{i=1}^S e^{(S-i)(L+4)m\eta} \\
	= & C_{T,m,K, L,A, d,\beta_{max}, |b(0,0)|}(\eta \delta)^{\frac{1}{2}}\left(1+|\ln \eta|+\frac{\delta}{\eta^{\frac{1}{4}}}\right)\left(\mathbb{E}|{\theta}_0|^4+1\right)^{\frac{7}{4}} \frac{1-e^{(L+4)m\eta S}}{1-e^{(L+4)m\eta}}\\
	\leq & C_{T,m,K, L,A, d,\beta_{max}, |b(0,0)|}(\eta \delta)^{\frac{1}{2}}\left(1+|\ln \eta|+\frac{\delta}{\eta^{\frac{1}{4}}}\right)\left(\mathbb{E}|{\theta}_0|^4+1\right)^{\frac{7}{4}}e^{(L+4)\eta T}.
\end{aligned}
$$
\end{proof}

\section{Conclusion}
In this paper, we construct a stochastic differential delay equation (SDDE) based on the DQN iteration and show that the weight of the action-value function in the DQN iteration is well-approximated by the solution of the SDDE in the Wasserstein-1 distance. More precisely, under appropriate smoothness assumptions on the Q-network, we establish an error bound for this approximation, proving that the approximation error converges to zero as the step size $\eta$ approaches zero.  

This result enables us to understand DQN’s two key techniques, the experience replay and the target network, from the perspective of continuous systems. On one hand, experience replay is essential for constructing the SDDE. With this technique, the corresponding SDDE has a unique solution, which also provides a continuous-time perspective on the convergence of the DQN algorithm. On the other hand, the target network technique corresponds to the delay term in the SDDE. Existing analyses of SDDEs show that such delays often reduce the fluctuations of a stochastic system. This perspective provides an intuitive explanation for why DQN reduces variance and enhances stability during the training of neural networks.

\section*{Acknowledgement}
We would like to gratefully thank Lihu Xu for the discussions and carefully revising the paper.

\appendix
\section{Deep Q-Network Algorithm}\label{appdix:algorithm}
The DQN algorithm is given below as Algorithm \ref{algorithm}.
\begin{algorithm}\label{algorithm}
	\caption{Deep Q-learning with experience replay}
	\label{algorithm}
	\KwIn{$\operatorname{MDP}(\mathcal{S}, \mathcal{A}, P, r, \gamma)$, replay memory $\mathcal{M}$, number of iterations $T$, minibatch size $n$, exploration probability $\epsilon \in(0,1)$, a family of deep Q-networks $Q_\theta: \mathcal{S} \times \mathcal{A} \rightarrow \mathbb{R}$, an integer $m$ for updating the target network, and a sequence of stepsizes $\left\{\eta_t\right\}_{t \geq 0}$.}
	Initialize the replay memory $\mathcal{M}$ to be empty.\\
	Initialize the Q-network with random weights $\theta$.\\
	Initialize the weights of the target network with $\theta^{-}=\theta$.\\
	Initialize the initial state $s_0$.\\
	\For{$t=0,1, \cdots, T$}{
		With probability $\epsilon$, choose $a_t$ uniformly at random from $\mathcal{A}$, and\\ with probability $1-\epsilon$, choose $a_t$ such that $Q_\theta\left(s_t, a_t\right)=\max _{a \in \mathcal{A}} Q_\theta\left(s_t, a\right)$\\
		Execute $a_t$ and observe reward $r_t$ and the next state $s_{t+1}$.\\
		Store transition $\left(s_t, a_t, r_t, s_{t+1}\right)$ in $\mathcal{M}$.\\
		Experience replay: Sample random minibatch of transitions $\left\{\left(s_i, a_i, r_i, s_i^{\prime}\right)\right\}_{i \in[H]}$ from $\mathcal{M}$.\\
		For each $i \in[H]$, compute the target $Y_i=r_i+\gamma \cdot \max _{a \in \mathcal{A}} Q_{\theta^{-}}\left(s_i^{\prime}, a\right)$.\\
		Update the Q-network: Perform a gradient descent step
		$$
		\theta \leftarrow \theta+\eta_t \cdot \frac{1}{H} \sum_{i \in[H]}\left[Y_i-Q_\theta\left(s_i, a_i\right)\right] \cdot \nabla_\theta Q_\theta\left(s_i, a_i\right) .
		$$
		Update the target network: Update $\theta^{-} \leftarrow \theta$ every $m$ steps.}
	Define policy ${\pi}$ as the greedy policy with respect to $Q_\theta$.
	
	\KwOut{Action-value function $Q_\theta$ and policy ${\pi}$.}
\end{algorithm}

\section{Proof of lemmas in section \ref{Result}}{\label{AAA}}

At first, we give following lemma, which will be useful. 
\begin{lemma}{\label{lem1}}
	Under Assumption \ref{assum:2} (b), the following map is continuous and Lipschitz continuous in the $\theta$-coordinate:
	$$
	(s,a,\theta)\rightarrow \int \max _{a^{\prime} \in \mathcal{A}} Q\left(s^{\prime}, a^{\prime} ; \theta\right) p\left(d s^{\prime} \mid s, a\right).
	$$
\end{lemma}
\begin{proof}
	We begin by fixing arbitrary $\hat{s} \in \mathcal{S}$ and $\hat{a} \in \mathcal{A}$. Given $\theta \in \mathbb{R}^d$, Assumption \ref{assum:2} (b) implies the existence of a constant $C$, such that $\forall \theta_1, \theta_2\in \mathbb{R}^d$ :
	$$
	\left|Q\left(\hat{s}, \hat{a} ; \theta_1\right)-Q\left(\hat{s}, \hat{a} ; \theta_2\right)\right| \leq C\left|\theta_1-\theta_2\right|.
	$$
	
	If $\max _{a^{\prime} \in \mathcal{A}} Q(s^{\prime}, a^{\prime} ; \theta_1)\geq\max _{a^{\prime} \in \mathcal{A}} Q(s^{\prime}, a^{\prime} ; \theta_2)$. Define $a_1(s^{\prime}):=\underset{a^{\prime} \in \mathcal{A}}{\operatorname{argmax}} Q(s^{\prime}, a^{\prime} ; \theta_1)$, we can get,
	\begin{align*}
		|\max _{a^{\prime} \in \mathcal{A}} Q(s^{\prime}, a^{\prime} ; \theta_1)-\max _{a^{\prime} \in \mathcal{A}} Q(s^{\prime}, a^{\prime} ; \theta_2)| \leq |Q\left(s^{\prime}, a_1(s^{\prime}) ; \theta_1\right)-Q\left(s^{\prime}, a_1(s^{\prime}) ; \theta_2\right)|\leq C\left|\theta_1-\theta_2\right|.
	\end{align*}
	For $\max _{a^{\prime} \in \mathcal{A}} Q(s^{\prime}, a^{\prime} ; \theta_1)\leq\max _{a^{\prime} \in \mathcal{A}} Q(s^{\prime}, a^{\prime} ; \theta_2)$. Similarly, define $a_2(s^{\prime}):=\underset{a^{\prime} \in \mathcal{A}}{\operatorname{argmax}} Q(s^{\prime}, a^{\prime} ; \theta_2)$, it is easy to know that,
	\begin{align*}
		|\max _{a^{\prime} \in \mathcal{A}} Q(s^{\prime}, a^{\prime} ; \theta_1)-\max _{a^{\prime} \in \mathcal{A}} Q(s^{\prime}, a^{\prime} ; \theta_2)|&=|\max _{a^{\prime} \in \mathcal{A}} Q(s^{\prime}, a^{\prime} ; \theta_2)-\max _{a^{\prime} \in \mathcal{A}} Q(s^{\prime}, a^{\prime} ; \theta_1)|\\
		&\leq |Q\left(s^{\prime}, a_2(s^{\prime}) ; \theta_2\right)-Q\left(s^{\prime}, a_2(s^{\prime}) ; \theta_1\right)|\\
		&\leq C\left|\theta_1-\theta_2\right|.
	\end{align*}
	Hitherto presented arguments and observations yield:
	$$
	\left|\int \max _{a^{\prime} \in \mathcal{A}} Q\left(s^{\prime}, a^{\prime} ; \theta_1\right) p\left(d s^{\prime} \mid s,a\right)- \int \max _{a^{\prime} \in \mathcal{A}} Q\left(s^{\prime}, a^{\prime} ; \theta_2\right) p\left(d s^{\prime} \mid s, a\right)\right| 
	\leq C\left|\theta_1-\theta_2\right|.
	$$
\end{proof}
\begin{remark}{\label{remarkbarQ}}
	From Assumption \ref{assum:2} (a), it is easy to know that $\int \max _{a^{\prime} \in \mathcal{A}} Q(s^{\prime}, a^{\prime} ; \theta) p(d s^{\prime} \mid s, a)$ is bounded.
\end{remark}
\subsection{Proof of lemma \ref{property1}}

(i) It is easy to get that,
\begin{align}{\label{Lipschitz0}}
	&\mathbb E|b_n\left(x_1, y_1\right)-b_n\left(x_2, y_2\right)|\\
	= & \mathbb E_{q,p}\left| \nabla Q\left(s,a;x_1\right)\left(R(s,a)+\gamma\max _{a^\prime \in \mathcal{A}} Q(s^\prime, a^\prime ; y_1) -Q\left(s,a;x_1\right)\right)\right.\nonumber\\
	&\left.-\nabla Q\left(s,a;x_2\right)\left(R(s,a)+\gamma \max _{a^\prime \in \mathcal{A}} Q(s^\prime, a^\prime ; y_2)-Q(s,a;x_2)\right)\right| \nonumber\\
	\leq & \mathbb E_{q,p}\left|\left[\nabla Q\left(s,a;x_1\right)-\nabla Q\left(s,a;x_2\right)\right] R(s,a)+\left[\nabla Q\left(s,a;x_1\right) \max _{a^\prime \in \mathcal{A}} Q(s^\prime, a^\prime ; y_1)\right.\right.\nonumber\\
	&\left.\left.-\nabla Q\left(s,a;x_2\right) \max _{a^\prime \in \mathcal{A}} Q(s^\prime, a^\prime ; y_2)\right] \gamma-\left[\nabla Q(s,a;x_1) Q(s,a;x_1)-\nabla Q(s,a;x_2)Q(s,a;x_2)\right]\right| \nonumber\\
	\leq & \sup_{s \in \mathcal{S}, a \in \mathcal{A}} C\left[(|R(s,a)|+|\max _{a^\prime \in \mathcal{A}} Q(s^\prime, a^\prime ; y_2)|+|Q(s,a;x_1)|+|\nabla Q\left(s,a;x_2\right)|)\left|x_1-x_2\right|\right.\nonumber\\
	&\left.+|\nabla Q\left(s,a;x_1\right)|\cdot|y_1-y_2|\right]\nonumber\\
	\leq& L(|x_1-x_2|+|y_1-y_2|),\nonumber
\end{align}
where the next to last inequality comes from Assumption \ref{assum:2} (b), i.e. $Q(s,a; \theta)$ is twice bounded continuous differentiable, and Lemma \ref{lem1}, i.e. $\max _{a \in \mathcal{A}} Q(s, a ; \theta)$ is Lipschitz continuous in the $\theta$ coordinate, the last inequality comes from Assumption  \ref{assum:2} (a) (b), \ref{assum:1} and Remark \ref{remarkbarQ}.

Then, we can get
\begin{align*}
	\left|b\left(x_1, y_1\right)-b\left(x_2, y_2\right) \right|\leq \mathbb E|b_n\left(x_1, y_1\right)-b_n\left(x_2, y_2\right)|
	\leq L(|x_1-x_2|+|y_1-y_2|),
\end{align*}

(ii) It is easy to get that,
\begin{align}{\label{grad-L}}
	|b_n\left(x, y\right)|&=\left|\nabla Q(s, a ; x)\cdot(R(s, a)+\gamma \cdot\max _{a^\prime \in \mathcal{A}} Q(s^\prime, a^\prime ; y)-Q(s, a ;x))\right|\\
	&\leq \left[\left|\nabla Q\left(s,a;x\right)\right|(\left|R(s,a)\right|+\gamma\left| \max _{a^\prime \in \mathcal{A}} Q(s^\prime, a^\prime ; y)- Q\left(s,a;y\right)\right|\right.\nonumber\\ 
	&\qquad\left.+\gamma\left| Q\left(s,a;y\right)-Q\left(s,a;x\right)\right|+(1-\gamma)\left|Q\left(s,a;x\right)\right|)\right]\nonumber\\
	&\leq K(1+|x-y|),\nonumber
\end{align}
where $\left| \max _{a^\prime \in \mathcal{A}} Q(s^\prime, a^\prime ; y)- Q\left(s,a;y\right)\right|<C$ comes from Remark \ref{remarkQ}, i.e. $Q$ is continuous in $a$-coordinate and $s$-coordinate by Assumption \ref{assum:2} (c); $|R(s,a)|<C$ comes from Assumption \ref{assum:1}, and $| Q\left(s,a;y\right)-Q\left(s,a;x\right)|\leq C|x-y|$,  $\sup_{s \in \mathcal{S}, a \in \mathcal{A}}|Q\left(s,a;x\right)|\leq C$ can be obtained by Assumption \ref{assum:2} (a) and \ref{assum:1} respectively.

By (\ref{grad-L}), one can get
\begin{align}{\label{grad-L1}}
	\operatorname{tr}\left(\Sigma\left(x, y\right)\right)&\leq \mathbb{E}[|b_n(x, y)|^2]\leq \sup_{s \in \mathcal{S}, a \in \mathcal{A}}[\left|\nabla Q\left(s,a;x\right)\left(R(s,a)+\gamma \bar{Q}\left(s,a;y\right)-Q\left(s,a;x\right)\right)\right|]^2\\
	&\leq K^2(1+|x-y|)^2\nonumber,
\end{align}
which implies that
$$
\left\|\sigma\left(x, y\right)\right\|_{\mathrm{HS}} \leq K(1+\left|x-y\right|)+\sqrt{{\beta_{max}}}+\sqrt{\frac{\delta d}{\eta}} ,
$$
where $\beta_{max}=\max_{\theta}\|\bar \beta(\theta)\|_{\mathrm{HS}}$.
\begin{remark}
	The property of locally Lipschitz continuous of $b\left(x,y\right)$ implies
	\begin{align}{\label{Lipschitz1}}
		&|b\left(x,y\right)|\leq |b\left(0,y\right)|+L|x|,\quad
		|b\left(x,y\right)|\leq |b\left(x,0\right)|+L|y|,\\
		&|b\left(x,y\right)|\leq |b\left(0,0\right)|+L(|x|+|y|)\nonumber
	\end{align}
	and $\mathbb E|b_n\left(x, y\right)|$ also satisfies the above result, it is easy to verify that
	\begin{align}
		\left|\nabla_{v_1}b\left(x,y\right)\right| \leq L\left|v_1\right| .
	\end{align}
	By similar calculation of (\ref{grad-L}), and Remark \ref{remarkbarQ}, we can also get that,
	\begin{align}{\label{grad-L+}}
		|b\left(x,y\right)|&=\left|\mathbb E_{(s,a)\sim q} \nabla Q\left(s,a;x\right)\left(R(s,a)+\gamma \bar{Q}\left(s,a;y\right)-Q\left(s,a;x\right)\right)\right|\\
		&\leq \sup_{s \in \mathcal{S}, a \in \mathcal{A}}[|\nabla Q\left(s,a;x\right)\left(R(s,a)+\gamma \bar{Q}\left(s,a;y\right)-Q\left(s,a;x\right)\right)|]\nonumber\\
		&\leq \sup_{s \in \mathcal{S}, a \in \mathcal{A}}\left[|\nabla Q\left(s,a;x\right)|(|R(s,a)|+\gamma| \bar{Q}\left(s,a;y\right)- Q\left(s,a;y\right)|\right.\nonumber\\ 
		&\qquad\quad\left.+\gamma| Q\left(s,a;y\right)-Q\left(s,a;x\right)|+(1-\gamma)|Q\left(s,a;x\right)|)\right]\nonumber\\
		&\leq K(1+|x-y|).\nonumber
	\end{align}
\end{remark}
\subsection{Proof of lemma \ref{property2}}
Since 
$$b\left(x, y\right)=-\mathbb E_{(s,a)\sim q}\nabla Q\left(s, a; x\right)\cdot\left(R(s,a)+\gamma \cdot \bar Q\left(s,a; y\right)-Q\left(s,a; x\right)\right)$$
then, it is easy to calculate that,
\begin{align*}
	\nabla_{1,v_2} \nabla_{1,v_1} b\left(x,y\right)=&-\mathbb E_{(s, a)\sim q}\nabla_{v_2}\nabla_{v_1}\nabla Q\left(s, a; x\right)\cdot\left(R(s, a)+\gamma \cdot \bar Q\left(s, a ; y\right)-Q\left(s, a ; x\right)\right)\\
	&+\mathbb E_{(s, a)\sim q}\nabla_{v_1}\nabla Q\left(s, a; x\right)\cdot\left(\nabla_{v_2}Q\left(s, a ; x\right)\right)\\
	&+\mathbb E_{(s, a)\sim q}\nabla_{v_2}\nabla Q\left(s, a; x\right)\cdot\left(\nabla_{v_1}Q\left(s, a ; x\right)\right)\\
	&+\mathbb E_{(s, a)\sim q}\nabla Q\left(s, a; x\right)\cdot\left(\nabla_{v_2}\nabla_{v_1}Q\left(s, a ; x\right)\right)
\end{align*}
By Assumption \ref{assum:1} and \ref{assum:2} (a) (b), there exists a constant $A_1$,$A_2$, such that
$$\left|\nabla_{1,v_2} \nabla_{1,v_1} b\left(x,y\right)\right|\leq A_1, \quad\left|\nabla_{1,v_3} \nabla_{1,v_2} \nabla_{1,v_1} b\left(x,y\right)\right| \leq A_2$$
Since
$$\sigma\left(x, y\right):=\left[\Sigma(x, y)+\bar \beta(x)+\frac{\delta}{\eta}I_d\right]^{1/2}$$
It is easy to check that, under Assumption \ref{assum:1} and \ref{assum:2} (a) (b),
\begin{align*}
	\Sigma(x, y) &=\mathbb E\left[b_n(x, y)b_n(x, y)^T\right]-\left[b(x, y)b(x, y)^T\right],\\
	\bar \beta(x) &=\mathbb E_{(s, a)\sim q}[V(s, a)\nabla Q\left(s, a ; x\right)][V(s, a)\nabla Q\left(s, a ; x\right)]^T
\end{align*}
are both bounded continuous differentiable from 1st to 3rd order in the $x$-coordinate, and $\Sigma(x, y)$ is bounded continuous differentiable in the $y$-coordinate, then we can get (\ref{property2.2}).
\section{Proof of lemmas in section \ref{proof}}{\label{B}}
\subsection{Proof of Lemma \ref{Markov}}
(i) For an $s \in \mathbb{Z}^{+}$, given $\tilde{\theta}_s$ (i.e. $\theta_{s m}$ ), by (\ref{original}) we know that the distribution of $\tilde{\theta}_{s+1}$ (i.e. $\theta_{(s+1) m}$ ) is uniquely determined by $\tilde{\theta}_s$ and the i.i.d. random variables $i_{s m+1}, \cdots, i_{s(m+1)}$, whence
$$
\mathbb{P}\left(\tilde{\theta}_{s+1} \in A \mid \tilde{\theta}_s, \cdots, \tilde{\theta}_0\right)=\mathbb{P}\left(\tilde{\theta}_{s+1} \in A \mid \tilde{\theta}_s\right), \quad A \in \mathcal{B}\left(\mathbb{R}^d\right) .
$$
So $(\tilde{\theta}_s)_{s \geq 0}$ is a Markov chain. Similarly, given $\tilde{X}_s$ (i.e. $X_{s m \eta}$ ), by (\ref{SDDE}), the distribution of $\tilde{X}_{s+1}$ is determined by $\tilde{X}_s$ and $\left(B_t\right)_{s m \eta \leq t \leq(s+1) m \eta}$, from which we know
$$
\mathbb{P}(\tilde{X}_{s+1} \in A \mid \tilde{X}_s, \cdots, \tilde{X}_1, \tilde{X}_0)=\mathbb{P}(\tilde{X}_{s+1} \in A \mid \tilde{X}_s), \quad A \in \mathcal{B}\left(\mathbb{R}^d\right),
$$
so $(\tilde{X}_s)_{s \geq 0}$ is a Markov chain.

(ii) The SDDE (\ref{SDDE}) restricted on the time period $[0, m \eta]$ reads as
\begin{equation}
	\mathrm{d}{X}_t=-b\left(X_t,X_0\right)\mathrm{d}t+\sqrt{\eta}\sigma(X_t,X_0)\mathrm{d}B_t, \text { for } t \in[0, m \eta]
\end{equation}
When $X_0=\theta_0$ is fixed, the above SDDE is equivalent to the following SDE:
\begin{equation}{\label{Internal}}
	\mathrm{d}{X}_t=-b\left(X_t,\theta_0\right)\mathrm{d}t+\sqrt{\eta}\sigma(X_t,\theta_0)\mathrm{d}B_t, \text { for } t \in[0, m \eta]
\end{equation}
thus is a time-homogeneous Markov process with states on $\mathbb{R}^d$. 

\subsection{Proof of Lemma \ref{lem3}}
For simplicity, denote $B(x):=-b\left(x,\theta_0\right)$ and $\sigma(x)=\sigma(x,\theta_0)$. Then, the SDE (\ref{Internal}) can be written as the following form:
\begin{align}{\label{M-SDE}}
	\mathrm{d} X_t=B\left(X_t\right) \mathrm{d} t+\sqrt{\eta} \sigma\left(X_t\right) \mathrm{d} B_t, \quad X_0=x,
\end{align}
where $B_t$ is a standard $d$-dimensional Brownian motion. 

Lemma \ref{property1}(i) and \ref{property2} can be rewritten as the following form:
\begin{lemma}{\label{A}}
	There exist constants $L \geq 0, A_i \geq 0$ with $i=1,2, \cdots, 5$, such that for any $x, y \in \mathbb{R}^d$ and unit vectors $v, v_1, v_2, v_3 \in \mathbb{R}^d$, we have
	\begin{align}
		\left|\nabla_v B(x)\right| &\leq L, \quad\left|\nabla_{v_2} \nabla_{v_1} B(x)\right| \leq A_1,\\
		\left|\nabla_{v_3} \nabla_{v_2} \nabla_{v_1} B(x)\right| &\leq A_2, \quad\left\|\nabla_{v_1} \sigma(x)\right\|_{\mathrm{HS}}^2 \leq A_3 \\
		\left\|\nabla_{v_1} \nabla_{v_2} \sigma(x)\right\|_{\mathrm{HS}}^2 &\leq A_4, \quad\left\|\nabla_{v_1} \nabla_{v_2} \nabla_{v_3} \sigma(x)\right\|_{\mathrm{HS}}^2 \leq A_5 .
	\end{align}
\end{lemma}
\begin{remark}
	Since $S(x)=\sigma(x) \sigma(x)^T=\Sigma(x)+\bar \beta (x)+\frac{\delta}{\eta} I_d$, $\Sigma(x)$ and $\bar \beta (x)$ are semi-positive definite, for any $0 \neq \xi \in \mathbb{R}^d$, we have
	\begin{align}
		\xi^T S(x) \xi \geq \frac{\delta}{\eta} \xi^T I_d \xi=\frac{\delta}{\eta}|\xi|^2 .
	\end{align}
\end{remark}
There exists a unique solution to the SDE (\ref{M-SDE}) by Lemma \ref{A}. According to the proof of Lemma 3.3 in \cite{chen2022approximation} we can get the result of Lemma \ref{lem3}.

\subsection{Proof of Lemma \ref{lem3.1}}
Recall (\ref{Internal}), by Itô's formula, we have 
\begin{align*}
	\frac{\mathrm{d}}{\mathrm{d} s} \mathbb{E}\left|X_s^x\right|^2 & =2 \mathbb{E}\left\langle X_s,-b\left(X_s, \theta_0\right)\right\rangle+\eta \mathbb{E}\left\|\sigma\left(X_s, \theta_0\right)\right\|_{\mathrm{HS}}^2 \\
	& \leq 2L\mathbb{E}\left|X_s^x\right|^2+ 2\mathbb{E}\left|X_s\right|\left|b\left(0, \theta_0\right)\right|+2 \eta \mathbb{E}\left(K^2\left|X_s-\theta_0\right|^2+K^2+\beta_{max}+\frac{\delta d}{\eta}\right) \\
	& \leq \left(2L+4 K^2 \eta+\frac{L}{2}\right) \mathbb{E}\left|X_s^x\right|^2+\frac{2}{L} \mathbb{E}\left|b\left(0, \theta_0\right)\right|^2+4 \eta K^2 \mathbb{E}\left|\theta_0\right|^2+2 \eta(K^2+\beta_{max})+2\delta d \\
	& \leq 3L \mathbb{E}\left|X_s^x\right|^2+\frac{2}{L} \mathbb{E}\left|b\left(0, \theta_0\right)\right|^2+4 \eta K^2 \mathbb{E}\left|\theta_0\right|^2+2 \eta(K^2+\beta_{max})+2\delta d.
\end{align*}
where the first inequality using (\ref{Lipschitz1}), (\ref{growth}), the last two inequality following Young's inequality and the fact $\eta<\frac{L}{8K^2}$.

Solving this differential inequality with initial data $X_0^x=x$ by Gronwall's inequality, (\ref{Lipschitz1}) and $t\leq m$, we can get
\begin{align*}
	\mathbb{E}\left|X_t^x\right|^2 & \leq e^{3L t}\left[|x|^2+\frac{2\left(L^{-1} \mathbb{E}\left|b\left(0, \theta_0\right)\right|^2+2 \eta K^2 \mathbb{E}\left|\theta_0\right|^2+\eta(K^2+\beta_{max})+\delta d\right)}{3L}\right] \\
	& \leq C_{K, L,d,m,\beta_{max}, |b(0,0)|}\left(1+|x|^2+\mathbb{E}\left|\theta_0\right|^2+\delta\right) .
\end{align*}

By the Cauchy-Schwarz inequality, Itô's isometry, (\ref{Lipschitz1}) and (\ref{growth}), we have
\begin{align*}
	\mathbb{E}\left|X_t^x-x\right|^2 \leq & 2 \mathbb{E}\left|\int_0^t b\left(X_r, \theta_0\right) \mathrm{d} r\right|^2+2 \mathbb{E}\left|\int_0^t \sqrt{\eta} \sigma\left(X_r, \theta_0\right) \mathrm{d} B_r\right|^2 \\
	\leq & 2 t \int_0^t \mathbb{E}\left|b\left(X_r, \theta_0\right)\right|^2 \mathrm{~d} r+2 \eta \int_0^t \mathbb{E}\left\|\sigma\left(X_r, \theta_0\right)\right\|_{\mathrm{HS}}^2 \mathrm{~d} r \\
	\leq & 4 t \int_0^t\left(|b(0,0)|^2+L^2 \mathbb{E}\left|X_r\right|^2+L^2 \mathbb{E}\left|\theta_0\right|^2\right) \mathrm{d} r \\
	& +4 \eta \int_0^t\left(K^2 \mathbb{E}\left|X_r-\theta_0\right|^2+K^2+\beta_{max}+\frac{\delta d}{\eta}\right) \mathrm{d} r \\
	\leq & 4\left(L^2 t+2 K^2 \eta\right) \int_0^t \mathbb{E}\left|X_r\right|^2 \mathrm{~d} r+4 t\left[t|b(0,0)|^2+L^2(t+2 \eta) \mathbb{E}\left|\theta_0\right|^2+\eta\beta_{max}+ \delta d\right].
\end{align*}
which, together with (\ref{lem3.1.1}), implies
$$
\mathbb{E}\left|X_t^x-x\right|^2 \leq C_{K, L,m,d,\beta_{max}, |b(0,0)|}\left(1+|x|^2+\mathbb{E}\left|\theta_0\right|^2+\delta\right) t(t+\eta+\delta).
$$
\subsection{Proof of Lemma \ref{lem3.2}}
By (\ref{original}), it is easy to see
\begin{align*}
	\mathbb{E}\left|\theta_n\right|^4= & \mathbb{E}\left|\theta_{n-1}\right|^4+\mathbb{E}\left|\eta b_n\left(\theta_{n-1}, \theta_0\right)-[\eta \beta_I(\theta_{n-1})+\sqrt{\eta\delta}I_d]W_n\right|^4 \\
	& -4 \mathbb{E}\left[\left|\theta_{n-1}\right|^2\left\langle\theta_{n-1}, \eta b_n\left(\theta_{n-1}, \theta_0\right)-[\eta \beta_I(\theta_{n-1})+\sqrt{\eta\delta}I_d]W_n\right\rangle\right] \\
	& +4 \mathbb{E}\left[\left\langle\theta_{n-1}, \eta b_n\left(\theta_{n-1}, \theta_0\right)-[\eta \beta_I(\theta_{n-1})+\sqrt{\eta\delta}I_d]W_n\right\rangle^2\right] \\
	& +2 \mathbb{E}\left[\left|\theta_{n-1}\right|^2\left|\eta b_n\left(\theta_{n-1}, \theta_0\right)-[\eta \beta_I(\theta_{n-1})+\sqrt{\eta\delta}I_d]W_n\right|^2\right] \\
	& -4 \mathbb{E}\left[|\eta b_n\left(\theta_{n-1}, \theta_0\right)-[\eta \beta_I(\theta_{n-1})+\sqrt{\eta\delta}I_d]W_n|^2\right.\\
	&\quad\left.\left\langle\theta_{n-1}, \eta b_n(\theta_{n-1}, \theta_0)-[\eta \beta_I(\theta_{n-1})+\sqrt{\eta\delta}I_d]W_n\right\rangle\right] .
\end{align*}
Now we estimate each term on the right hand side.

For the second term, the fact $\eta<(\frac{1}{432 L^3})^{1 / 3}$, (\ref{Lipschitz0}) and $\mathbb{E}|W|^4 \leq 3 d^2$ imply
\begin{align*}
	& \mathbb{E}\left|\eta b_n\left(\theta_{n-1}, \theta_0\right)-[\eta \beta_I(\theta_{n-1})+\sqrt{\eta\delta}I_d]W_n\right|^4 \\
	\leq & 8 \eta^4\left[\mathbb{E}\left|b_n\left(\theta_{n-1}, \theta_0\right)\right|^4\right]+8 \mathbb{E}\|\eta \beta_I(\theta_{n-1})+\sqrt{\eta\delta}I_d\|_{\operatorname{HS}}^4 \mathbb{E}\left|W\right|^4 \\
	\leq & 216 \eta^4\left[L^4 \mathbb{E}\left|\theta_{n-1}\right|^4+L^4 \mathbb{E}\left|\theta_0\right|^4+|b(0,0)|^4\right]+64 \left(\mathbb{E} \eta^4\|\beta_I(\theta_{n-1})\|_{\operatorname{HS}}^4 +(\eta\delta)^2\right)\mathbb{E}|W|^4 \\
	\leq & \frac{L}{2} \eta \mathbb{E}\left|\theta_{n-1}\right|^4+216 \eta^4\left(L^4 \mathbb{E}\left|\theta_0\right|^4+|b(0,0)|^4\right)+192(\eta^4\beta_{max}^4+(\eta \delta)^2) d^2.
\end{align*}

For the third term, since $W_n$ is independent of $i$, $\theta_{n-1}$, and the fact that $i$ is independent of $\theta_{n-1}$ and uniformly distributed, (\ref{Lipschitz}) yields,
\begin{align*}
	&-4 \mathbb{E}\left[\left|\theta_{n-1}\right|^2\left\langle\theta_{n-1}, \eta b_n\left(\theta_{n-1}, \theta_0\right)-[\eta \beta_I(\theta_{n-1})+\sqrt{\eta\delta}I_d]W_n\right\rangle\right]\\
	=& -4\mathbb{E}\left[\left|\theta_{n-1}\right|^2\left\langle\theta_{n-1}, \eta b_n\left(\theta_{n-1}, \theta_0\right)\right\rangle\right] \\
	= & -4\eta \mathbb{E}\left[\left|\theta_{n-1}\right|^2\left\langle b\left(\theta_{n-1}, \theta_0\right)-b\left(0, \theta_0\right), \theta_{n-1}\right\rangle\right]-4\eta \mathbb{E}\left[\left|\theta_{n-1}\right|^2\left\langle b\left(0, \theta_0\right), \theta_{n-1}\right\rangle\right] \\
	\leq & 4L\eta \mathbb{E}\left|\theta_{n-1}\right|^4-4\eta \mathbb{E}\left[\left|\theta_{n-1}\right|^2\left\langle b\left(0, \theta_0\right), \theta_{n-1}\right\rangle\right]\\
	\leq &  4L \eta \mathbb{E}\left|\theta_{n-1}\right|^4+4 \eta \mathbb{E}\left[\left|\theta_{n-1}\right|^3|b(0,0)|\right]+4 \eta \mathbb{E}\left[\left|\theta_{n-1}\right|^3\left|\theta_0\right|\right] \\
	\leq & 5L \eta \mathbb{E}\left|\theta_{n-1}\right|^4+\frac{216|b(0,0)|^4}{L^3} \eta+\frac{216\mathbb{E}|\theta_0|^4}{L^3} \eta.
\end{align*}

For the fourth term, (\ref{Lipschitz0}), Young's inequality and the fact $\eta<\frac{1}{64 L}$ implies

\begin{align*}
	& 4 \mathbb{E}\left[\left\langle\theta_{n-1}, \eta b_n\left(\theta_{n-1}, \theta_0\right)-[\eta \beta_I(\theta_{n-1})+\sqrt{\eta\delta}I_d]W_n\right\rangle^2\right] \\
	\leq & 8 \eta^2 \mathbb{E}[|\theta_{n-1}|^2(|b_n(\theta_{n-1}, \theta_0)|^2)]+8 \mathbb{E}[|\theta_{n-1}|^2\|\eta \beta_I(\theta_{n-1})+\sqrt{\eta\delta}I_d\|_{\mathrm{HS}}^2|W_n|^2] \\
	\leq & 8 \eta^2 \mathbb{E}\left[\left|\theta_{n-1}\right|^2\left(2L^2 \mathbb{E}\left|\theta_{n-1}\right|^2+2L^2 \mathbb{E}\left|\theta_0\right|^2+|b(0,0)|^2\right)\right]\\
	&+16(\eta^2\beta_{max}^2+ \eta \delta )d  \mathbb{E}\left[\left|\theta_{n-1}\right|^2\right] \\
	\leq & \frac{L}{2} \eta \mathbb{E}\left|\theta_{n-1}\right|^4+\frac{2}{L} \mathbb{E}\left[2\eta L^2 \mathbb{E}\left|\theta_0\right|^2+\eta|b(0,0)|^2+(\eta\beta_{max}^2+  \delta )d \right]^2 \\
	\leq & \frac{L}{2} \eta \mathbb{E}\left|\theta_{n-1}\right|^4+\frac{8}{L}\left(4L^4 \eta^2 \mathbb{E}\left|\theta_0\right|^4+\eta^2|b(0,0)|^4+(\eta\beta_{max}^2 d)^2+(\delta d)^2\right) .
\end{align*}

The fifth term can be estimated by a similar calculation with the fourth term, and we have
\begin{align*}
	& 2 \mathbb{E}\left[\left|\theta_{n-1}\right|^2\left|\eta b_n\left(\theta_{n-1}, \theta_0\right)-[\eta \beta_I(\theta_{n-1})+\sqrt{\eta\delta}I_d]W_n\right|^2\right] \\
	\leq &  \frac{L}{4} \eta \mathbb{E}\left|\theta_{n-1}\right|^4+\frac{4}{L}\left(4L^4 \eta^2 \mathbb{E}\left|\theta_0\right|^4+\eta^2|b(0,0)|^4+(\eta\beta_{max}^2d)^2+(\delta d)^2\right) .
\end{align*}

For the last term, by (\ref{Lipschitz0}), the Hölder inequality, Young's inequality and the fact $\eta<\left(\frac{3}{464 L^2}\right)^{\frac{1}{2}}$, we can get
\begin{align*}
	& 4 \mathbb{E}\left[|\eta b_n\left(\theta_{n-1}, \theta_0\right)-[\eta \beta_I(\theta_{n-1})+\sqrt{\eta\delta}I_d]W_n|^2\right.\\
	&\quad\left.\left\langle\theta_{n-1}, \eta b_n(\theta_{n-1}, \theta_0)-[\eta \beta_I(\theta_{n-1})+\sqrt{\eta\delta}I_d]W_n\right\rangle\right]\\
	& \leq 16 \mathbb{E}\left[\left[\eta^3\left|b_n\left(\theta_{n-1}, \theta_0\right)\right|^3+\left|[\eta \beta_I(\theta_{n-1})+\sqrt{\eta\delta}I_d]W_n\right|^3\right]\left|\theta_{n-1}\right|\right] \\
	& \leq 16 \eta^3 \mathbb{E}\left[\left|\theta_{n-1}\right|\left[4L^3 \mathbb{E}\left|\theta_{n-1}\right|^3+4L^3 \mathbb{E}\left|\theta_0\right|^3+|b(0,0)|^3\right]\right]+64[(\eta\beta_{max})^3+(\eta \delta)^{\frac{3}{2}}] \mathbb{E}\left[\left|\theta_{n-1}\right||W|^3\right] \\
	& \leq \frac{3L}{4} \eta \mathbb{E}\left|\theta_{n-1}\right|^4+12\left[ \eta^3\left(4L^3 \mathbb{E}\left|\theta_0\right|^4+\frac{1}{L^{3/4}}|b(0,0)|^4\right)+4(\eta\beta_{max})^4d^2+12(\eta\delta d)^2\right] .
\end{align*}
Since $\eta<1$, the inequalities above imply
\begin{align}{\label{thetan}}
	\mathbb{E}\left|\theta_n\right|^4 \leq(1+7L \eta) \mathbb{E}\left|\theta_{n-1}\right|^4+C_{L}\left(|b(0,0)|^4+\mathbb{E}\left|\theta_0\right|^4+\beta_{max}^4d^2+\delta^2 d^2\right) \eta .
\end{align}
Therefore, by Gronwall's inequality, $\delta\leq 1$ and the fact $n\leq m$,
\begin{align*}
	\mathbb{E}\left|\theta_n^x\right|^4 & \leq(1+7L \eta)^n|x|^4+C_{ L}\left(|b(0,0)|^4+\mathbb{E}\left|\theta_0\right|^4+\beta_{max}^4d^2+\delta^2 d^2\right) \eta \sum_{j=0}^{n-1}(1+7L \eta)^j \\
	& \leq  C_{L,m, d,\beta_{max}, |b(0,0)|}\left(1+|x|^4+\mathbb{E}\left|\theta_0\right|^4\right).
\end{align*}
\subsection{Proof of Lemma \ref{lem4}}
For any $u_t(x)=\mathbb{E} h\left(X_t^x\right)$ with $k \geq 1$, by (\ref{generate2}), we have
\begin{align*}
	& \mathbb{E} \int_0^1 \mathcal{A}^Z u_t\left(X_{\eta s}^x\right) \mathrm{d} s \\
	= & -\eta \mathbb{E} \int_0^1\left\langle b\left(X_{\eta s}^x,\theta_0\right), \nabla u_t(X_{\eta s}^x)\right\rangle \mathrm{d} s+\frac{1}{2} \eta \mathbb{E} \int_0^1\left\langle\eta\Sigma(X_{\eta s}^x, \theta_0)+\eta\bar \beta(X_{\eta s}^x)+{\delta}I_d, \nabla^2 u_t(X_{\eta s}^x)\right\rangle_{\mathrm{HS}} \mathrm{d} s \\
	= &  - \mathbb{E} \int_0^\eta\left\langle b\left(X_{ s}^x,\theta_0\right), \nabla u_t(X_{s}^x)\right\rangle \mathrm{d} s+\frac{1}{2} \mathbb{E} \int_0^\eta\left\langle\eta\Sigma(X_{s}^x, \theta_0)+\eta\bar \beta(X_{ s}^x)+{\delta}I_d, \nabla^2 u_t(X_{ s}^x)\right\rangle_{\mathrm{HS}} \mathrm{d} s .
\end{align*}
By (\ref{generate1}), we have
\begin{align*}
	\mathcal{A}^\theta u_t(x) =\mathbb{E}\left[u_t\left(x-\eta b_n\left(x, \theta_0\right)+(\eta \beta_I\left(x\right)+\sqrt{\eta \delta}I_d) W\right)\right]-u_t(x).
\end{align*}
Then, by Taylor's expansion, we have
\begin{align*}
	\mathcal{A}^\theta u_t(x)= & \mathbb{E}\left[\langle\nabla u_t(x), -\eta b_n(x, \theta_0)+(\eta \beta_I(x)+\sqrt{\eta \delta}I_d) W\rangle\right]\\
	& +\frac{1}{2} \mathbb{E}\left\langle\nabla^2 u_t(x),[-\eta b_n(x, \theta_0)+(\eta \beta_I(x)+\sqrt{\eta \delta}I_d) W]\right. \\
	& \left.\quad[-\eta b_n(x, \theta_0)+(\eta \beta_I(x)+\sqrt{\eta \delta}I_d) W]^T\right\rangle_{\mathrm{HS}}+\mathbb{E}[\mathcal{R}^{u_t}(x)]  \\
	= & \left\langle\nabla u_t(x),-\eta b(x, \theta_0)\right\rangle+\frac{1}{2} \eta^2\left\langle\nabla^2 u_t(x), \left[b(x, \theta_0)\right]^2+\mathbb E[\sigma (x,\theta_0)]^2\right\rangle_{\mathrm{HS}} +\mathbb{E}\left[\mathcal{R}^{u_t}(x)\right],
\end{align*}
where
\begin{align*}
	\mathcal{R}^{u_t}(x) &=\int_0^1 \int_0^r\left\langle\nabla^2 u_t(x+s[-\eta b_n(x, \theta_0)+(\eta \beta_I(x)+\sqrt{\eta \delta}I_d) W])-\nabla^2 u_t(x),\right. \\
	&\quad {\left.[-\eta b_n(x, \theta_0)+(\eta \beta_I(x)+\sqrt{\eta \delta}I_d) W][-\eta b_n(x, \theta_0)+(\eta \beta_I(x)+\sqrt{\eta \delta}I_d) W]^T\right\rangle \mathrm{d} s \mathrm{~d} r . }
\end{align*}
Therefore, we have
$$
\left|\mathbb{E} \int_0^1\left[\mathcal{A}^Z u_t\left(Z_s^x\right)-\mathcal{A}^\theta u_t(x)\right] \mathrm{d} s\right| \leq \mathcal{J}_1+\mathcal{J}_2+\mathbb{E}\left|\mathcal{R}^{u_t}(x)\right|,
$$
where
\begin{align*}
	\mathcal{J}_1:=&\left|\mathbb{E} \int_0^\eta\left\langle\nabla u_t\left(X_s^x\right), b\left(X_{ s}^x,\theta_0\right)\right\rangle \mathrm{d} s  -\eta\left\langle\nabla u_t(x), b(x, \theta_0)\right\rangle\right.\\
	&\quad\quad\quad\left.+\frac{1}{2} \eta^2\left\langle\nabla^2 u_t(x), b(x, \theta_0)(b(x, \theta_0))^T\right\rangle_{\mathrm{HS}} \right|\\
	\mathcal{J}_2:=&\left|\frac{1}{2} \mathbb{E} \int_0^\eta\left\langle\eta\Sigma(X_{s}^x, \theta_0)+\eta\bar \beta(X_{ s}^x) +{\delta}I_d, \nabla^2 u_t(X_{ s}^x)\right\rangle_{\mathrm{HS}} \mathrm{d} s\right.\\
	&\quad\quad\quad\left.-\frac{1}{2} \eta^2\left\langle\nabla^2 u_t(x),\left[b(x, \theta_0)\right]^2+\mathbb E[\sigma (x,\theta_0)]^2\right\rangle_{\mathrm{HS}}\right|
\end{align*}

For $\mathcal{J}_1$, we have
\begin{align*}
	\mathcal{J}_1 \leq & \left|\mathbb{E} \int_0^\eta\left\langle\nabla u_t\left(X_s^x\right), b\left(X_{ s}^x,\theta_0\right)-b(x, \theta_0)\right\rangle \mathrm{d} s\right| \\
	& +\left|\mathbb{E} \int_0^\eta\left\langle\nabla u_t\left(X_s^x\right)-\nabla u_t(x), b(x, \theta_0)\right\rangle \mathrm{d} s+\frac{1}{2} \eta^2\left\langle\nabla^2 u_t(x), b(x, \theta_0)(b(x, \theta_0))^T\right\rangle_{\mathrm{HS}}\right| \\
	:= & \mathcal{J}_{11}+\mathcal{J}_{12} .
\end{align*}
As for $\mathcal{J}_{11}$, by (\ref{P1}), (\ref{Lipschitz}), the Cauchy-Schwarz inequality and (\ref{lem3.1.2}), one has
\begin{align*}
	\mathcal{J}_{11} & \leq C_{m,L} \int_0^\eta \mathbb{E}\left|X_s^x-x\right| \mathrm{d} s \\
	&\leq C_{K, L,m, d,\beta_{max}, |b(0,0)|}\int_0^\eta\left(1+|x|+\sqrt{\mathbb{E}\left|\theta_0\right|^2}+\delta^{\frac{1}{2}}\right) \sqrt{s(s+\eta+\delta)}\mathrm{d} s.\\
	& \leq C_{K, L,m, d,\beta_{max}, |b(0,0)|}\left(1+|x|+\sqrt{\mathbb{E}\left|\theta_0\right|^2}+\delta^{\frac{1}{2}}\right) \eta^{\frac{3}{2}}\left(\eta^{\frac{1}{2}}+\delta^{\frac{1}{2}}\right) .
\end{align*}
As for $\mathcal{J}_{12}$, since
\begin{align*}
	& \mathbb{E}\left\langle\nabla u_t\left(X_s^x\right)-\nabla u_t(x), b(x, \theta_0)\right\rangle \\
	= & \mathbb{E}\left\langle\nabla^2 u_t(x),\left(X_s^x-x\right)(b(x, \theta_0))^T\right\rangle_{\mathrm{HS}} \\
	& +\int_0^1 \mathbb{E}\left\langle\nabla^2 u_t\left(x+r\left(X_s^x-x\right)\right)-\nabla^2 u_t(x),\left(X_s^x-x\right)(b(x, \theta_0))^T\right\rangle_{\mathrm{HS}} \mathrm{d} r \\
	= & -\int_0^s \mathbb{E}\left\langle\nabla^2 u_t(x), b\left(X_v^x, \theta_0\right)(b(x, \theta_0))^T\right\rangle_{\mathrm{HS}} \mathrm{d} v \\
	& +\int_0^1 \mathbb{E}\left\langle\nabla^2 u_t\left(x+r\left(X_s^x-x\right)\right)-\nabla^2 u_t(x),\left(X_s^x-x\right)(b(x, \theta_0))^T\right\rangle_{\mathrm{HS}} \mathrm{d} r .
\end{align*}
By (\ref{P2}), (\ref{P3}) and (\ref{Lipschitz}), we have
\begin{align*}
	\mathcal{J}_{12} \leq & \left|\mathbb{E} \int_0^\eta \int_0^s \mathbb{E}\left\langle\nabla^2 u_t(x),\left(b\left(X_v^x, \theta_0\right)-b(x, \theta_0)\right)(b(x, \theta_0))^T\right\rangle_{\mathrm{HS}} \mathrm{d} v \mathrm{~d} s\right| \\
	& +\left|\int_0^\eta \int_0^1 \mathbb{E}\left\langle\nabla^2 u_t\left(x+r\left(X_s^x-x\right)\right)-\nabla^2 u_t(x),\left(X_s^x-x\right)(b(x, \theta_0))^T\right\rangle_{\mathrm{HS}} \mathrm{d} r \mathrm{~d} s\right| \\
	\leq & C_{A, d,|b\left(0, 0\right)|, L}\left(1+\frac{1}{\sqrt{\delta t}}\right)(1+|x|+\mathbb{E}|\theta_0|) \int_0^\eta \int_0^s \mathbb{E}\left|X_v^x-x\right| \mathrm{d} v \mathrm{~d} s \\
	& +C_{A, d,|b\left(0, 0\right)|, L}\left(1+\frac{1}{\delta t}+\frac{1}{t^{\frac{5}{4}}}\right)(1+|x|+\mathbb{E}|\theta_0|) \int_0^\eta \int_0^1 r \mathbb{E}\left|X_s^x-x\right|^2 \mathrm{~d} r \mathrm{~d} s .
\end{align*}
Then, by the Cauchy-Schwarz inequality, (\ref{lem3.1.2}) and the condition $\eta \leq \delta \leq 1$, we can get
\begin{align*}
	\mathcal{J}_{12} \leq& C_{A,K, L,m, d,\beta_{max}, |b(0,0)|}\left(1+\frac{1}{\sqrt{\delta t}}\right)\left(1+|x|^2+\mathbb{E}\left|\theta_0\right|^2\right) \eta^{\frac{5}{2}}\left(\eta^{\frac{1}{2}}+\delta^{\frac{1}{2}}\right)\\
	&+C_{A,K, L,m, d,\beta_{max}, |b(0,0)|}\left(1+\frac{1}{\delta t}+\frac{1}{t^{\frac{5}{4}}}\right)(1+|x|+\mathbb{E}|\theta_0|)\left(1+|x|^2+\mathbb{E}\left|\theta_0\right|^2\right) \eta^2(\eta+\delta)\\
	\leq& C_{A,K, L,m, d,\beta_{max}, |b(0,0)|}\left(1+\frac{1}{t}+\frac{\delta}{t^{\frac{5}{4}}}\right)(1+|x|+\mathbb{E}|\theta_0|)\left(1+|x|^2+\mathbb{E}\left|\theta_0\right|^2\right) \eta^2
\end{align*}
Hence,
$$
\mathcal{J}_1 \leq C_{A,K, L,m, d,\beta_{max}, |b(0,0)|}(1+|x|+\mathbb{E}|\theta_0|)\left(1+|x|^2+\mathbb{E}\left|\theta_0\right|^2\right)\left[\left(\frac{1}{t}+\frac{\delta}{t^{\frac{5}{4}}}\right) \eta^{\frac{1}{2}}+\delta^{\frac{1}{2}}\right] \eta^{\frac{3}{2}} .
$$

For $\mathcal{J}_2$, notice that $\eta \mathbb{E}\left[\sigma\left(x, \theta_0\right)\right]^2=\eta \mathbb{E}\left[\Sigma\left(x, \theta_0\right)\right]+\eta\mathbb{E}\left[\bar \beta(x)\right]+\delta I_d$, and for $x, y, z \in \mathbb{R}^d$, following the definition of $\Sigma(x, y)$, a straight calculation gives that
$$
\begin{aligned}
	&\Sigma(x, y)-\Sigma(z, y) \\
	=& \mathbb E\left[b_n(x, y)b_n(x, y)^T\right]-\left[b(x, y)b(x, y)^T\right]\\
	&- \mathbb E\left[b_n(z, y)b_n(z, y)^T\right]+\left[b(z, y)b(z, y)^T\right]\\
	= & \mathbb E\left[(b_n(x, y)-b_n(z, y))b_n(x, y)^T\right]+ \mathbb E\left[b_n(z, y)(b_n(x, y)-b_n(z, y))^T\right]\\
	&-\left[(b(x, y)-b(z, y)) b(x, y)^T\right]-\left[b(z, y)(b(x, y)-b(z, y))^T\right]
\end{aligned}
$$
By (\ref{Lipschitz}), and (\ref{grad-L}), we further have
$$
\|\Sigma(x, y)-\Sigma(z, y)\|_{\mathrm{HS}} \leq 2 LK(1+|x-y|+|z-y|)|x-z| \leq 2 LK(1+|x|+2|y|+|z|)|x-z| .
$$
Then, the Cauchy-Schwarz inequality, (\ref{grad-L1}), (\ref{P2}) and (\ref{P3}) imply
\begin{align*}
	\mathcal{J}_2 \leq & \frac{\eta}{2} \mathbb{E}\left|\int_0^\eta\left\langle\nabla^2 u_t\left(X_s^x\right), \Sigma\left(X_s^x, \theta_0\right)-\Sigma\left(x, \theta_0\right)\right\rangle_{\mathrm{HS}} \mathrm{d} s\right| \\
	& +\frac{1}{2} \mathbb{E}\left|\int_0^\eta\left\langle\nabla^2 u_t\left(X_s^x\right)-\nabla^2 u_t(x), \eta \Sigma\left(x, \theta_0\right)+\eta\bar\beta(x)+\delta I_d\right\rangle_{\mathrm{HS}} \mathrm{d} s\right| \\
	\leq & \eta C_{A,K, L, d}\left(1+\frac{1}{\sqrt{\delta t}}\right) \int_0^\eta \mathbb{E}\left[\left(1+\left|X_s^x\right|+\left|\theta_0\right|+|x|\right)\left|X_s^x-x\right|\right] \mathrm{d} s \\
	& +C_{A,K, L,d,\beta_{max}}\left(1+\frac{1}{\delta t}+\frac{1}{t^{\frac{5}{4}}}\right) \int_0^\eta \mathbb{E}\left[\left|X_s^x-x\right|\left(\eta+\eta|x|^2+\eta\left|\theta_0\right|^2+\delta\right)\right] \mathrm{d} s,
\end{align*}
By the Cauchy-Schwarz inequality, (\ref{lem3.1.1}) and (\ref{lem3.1.2}), one has
$$
\begin{gathered}	
	\mathcal{J}_2 \leq C_{A,K, L,m, d,\beta_{max}, |b(0,0)|}\left(1+\frac{1}{\sqrt{\delta t}}\right)\left(1+|x|^2+\mathbb{E}\left|\theta_0\right|^2+\delta\right) \eta^{\frac{5}{2}}\left(\eta^{\frac{1}{2}}+\delta^{\frac{1}{2}}\right) \\
	+C_{A,K, L,m, d,\beta_{max}, |b(0,0)|}\left(1+\frac{1}{\delta t}+\frac{1}{t^{\frac{5}{4}}}\right)\left(1+\sqrt{\mathbb{E}\left|\theta_0\right|^2}+\delta^{\frac{1}{2}}\right) \\
	\left(1+\sqrt{\mathbb{E}\left|\theta_0\right|^4}\right)\left(1+|x|^3\right) \eta^{\frac{3}{2}}\left(\eta^{\frac{1}{2}}+\delta^{\frac{1}{2}}\right)(\eta+\delta).
\end{gathered}
$$
The condition $\eta \leq \delta \leq 1$ further implies
$$
\mathcal{J}_2 \leq C_{A,K, L,m, d,\beta_{max}, |b(0,0)|}\left(1+\frac{1}{t}+\frac{\delta}{t^{\frac{5}{4}}}\right)\left(1+\mathbb{E}\left|\theta_0\right|^4\right)\left(1+|x|^3\right) \eta^{\frac{3}{2}} \delta^{\frac{1}{2}} $$

For $\mathbb{E}\left|\mathcal{R}^{u_t}(x)\right|$, by (\ref{P3}), (\ref{Lipschitz0}) and Hölder's inequality, we have
\begin{align*}
	\mathbb{E}\left|\mathcal{R}^{u_t}(x)\right|
	& \leq C_{A, L, d}\left(1+\frac{1}{\delta t}+\frac{1}{t^{\frac{5}{4}}}\right) \mathbb{E}\left|-\eta b_n(x, \theta_0)+(\eta \beta_I(x)+\sqrt{\eta \delta}I_d) W\right|^3 \\
	& \leq C_{A,K, L,m, d,\beta_{max}, |b(0,0)|}\left(1+\frac{1}{\delta t}+\frac{1}{t^{\frac{5}{4}}}\right)\left[\eta^3\left(1+|x|^3+\mathbb{E}\left|\theta_0\right|^3\right)+(\eta \delta)^{\frac{3}{2}}\right] \\
	& \leq C_{A,K, L,m, d,\beta_{max}, |b(0,0)|}\left(1+\frac{1}{t}+\frac{\delta}{t^{\frac{5}{4}}}\right)\left(1+|x|^3+\mathbb{E}\left|\theta_0\right|^3\right) \eta^{\frac{3}{2}} \delta^{\frac{1}{2}} .
\end{align*}
Combining all of above, we have
\begin{align*}
	& \left|\mathbb{E} \int_0^1\left[\mathcal{A}^Z u_t\left(Z_s^x\right)-\mathcal{A}^\theta u_t(x)\right] \mathrm{d} s\right| \\
	\leq & C_{A,K, L,m, d,\beta_{max}, |b(0,0)|}\left(1+\frac{1}{t}+\frac{\delta}{t^{\frac{5}{4}}}\right)\left(1+\mathbb{E}\left|\theta_0\right|^4\right)\left(1+|x|^3\right) \eta^{\frac{3}{2}} \delta^{\frac{1}{2}}.
\end{align*}

\bibliographystyle{amsplain}
\bibliography{Approximation_to_Deep_Q-Network_by_Stochastic_Delay_Differential_Equations}

\providecommand{\bysame}{\leavevmode\hbox to3em{\hrulefill}\thinspace}
\providecommand{\MR}{\relax\ifhmode\unskip\space\fi MR }
\providecommand{\MRhref}[2]{%
  \href{http://www.ams.org/mathscinet-getitem?mr=#1}{#2}
}
\providecommand{\href}[2]{#2}
\begin{thebibliography}{10}

\bibitem{bertsekas2019reinforcement}
Dimitri Bertsekas, \emph{Reinforcement learning and optimal control}, Athena
  Scientific, 2019.

\bibitem{biyik2023active}
Erdem B{\i}y{\i}k, Nicolas Huynh, Mykel~J Kochenderfer, and Dorsa Sadigh,
  \emph{Active preference-based gaussian process regression for reward learning
  and optimization}, The International Journal of Robotics Research (2023),
  02783649231208729.

\bibitem{MR4725527}
Qi~Cai, Zhuoran Yang, Jason~D. Lee, and Zhaoran Wang, \emph{Neural temporal
  difference and {Q} learning provably converge to global optima}, Math. Oper.
  Res. \textbf{49} (2024), no.~1, 619--651. \MR{4725527}

\bibitem{carvalho2020new}
Diogo Carvalho, Francisco~S Melo, and Pedro Santos, \emph{A new convergent
  variant of q-learning with linear function approximation}, Advances in Neural
  Information Processing Systems \textbf{33} (2020), 19412--19421.

\bibitem{MR4409807}
Peng Chen, Jianya Lu, and Lihu Xu, \emph{Approximation to stochastic variance
  reduced gradient {L}angevin dynamics by stochastic delay differential
  equations}, Appl. Math. Optim. \textbf{85} (2022), no.~2, Paper No. 15, 40.
  \MR{4409807}

\bibitem{chen2022approximation}
\bysame, \emph{Approximation to stochastic variance reduced gradient langevin
  dynamics by stochastic delay differential equations}, Applied Mathematics \&
  Optimization \textbf{85} (2022), no.~2, 15.

\bibitem{chen2023probability}
Peng Chen, Qi-Man Shao, and Lihu Xu, \emph{A probability approximation
  framework: Markov process approach}, The Annals of Applied Probability
  \textbf{33} (2023), no.~2, 1619--1659.

\bibitem{MR4095347}
Xi~Chen, Simon~S. Du, and Xin~T. Tong, \emph{On stationary-point hitting time
  and ergodicity of stochastic gradient {L}angevin dynamics}, J. Mach. Learn.
  Res. \textbf{21} (2020), Paper No. 68, 41. \MR{4095347}

\bibitem{fan2020theoretical}
Jianqing Fan, Zhaoran Wang, Yuchen Xie, and Zhuoran Yang, \emph{A theoretical
  analysis of deep q-learning}, Learning for dynamics and control, PMLR, 2020,
  pp.~486--489.

\bibitem{guillin2024error}
Arnaud Guillin, Yu~Wang, Lihu Xu, and Haoran Yang, \emph{Error estimates
  between sgd with momentum and underdamped langevin diffusion}, arXiv preprint
  arXiv:2410.17297 (2024).

\bibitem{ishfaq2021randomized}
Haque Ishfaq, Qiwen Cui, Viet Nguyen, Alex Ayoub, Zhuoran Yang, Zhaoran Wang,
  Doina Precup, and Lin Yang, \emph{Randomized exploration in reinforcement
  learning with general value function approximation}, International Conference
  on Machine Learning, PMLR, 2021, pp.~4607--4616.

\bibitem{MR4279772}
Vikram Krishnamurthy and George Yin, \emph{Langevin dynamics for adaptive
  inverse reinforcement learning of stochastic gradient algorithms}, J. Mach.
  Learn. Res. \textbf{22} (2021), Paper No. 121, 49. \MR{4279772}

\bibitem{li2017stochastic}
Qianxiao Li, Cheng Tai, and E~Weinan, \emph{Stochastic modified equations and
  adaptive stochastic gradient algorithms}, International Conference on Machine
  Learning, PMLR, 2017, pp.~2101--2110.

\bibitem{li2019stochastic}
\bysame, \emph{Stochastic modified equations and dynamics of stochastic
  gradient algorithms i: Mathematical foundations}, The Journal of Machine
  Learning Research \textbf{20} (2019), no.~1, 1474--1520.

\bibitem{liu2022understanding}
Fanghui Liu, Luca Viano, and Volkan Cevher, \emph{Understanding deep neural
  function approximation in reinforcement learning via $\epsilon$-greedy
  exploration}, [Proceedings of NeurIPS 2022] (2022).

\bibitem{miyamoto2021convergence}
Konatsu Miyamoto, Masaya Suzuki, Yuma Kigami, and Kodai Satake,
  \emph{Convergence of q-value in case of gaussian rewards}, Progress in
  Intelligent Decision Science: Proceeding of IDS 2020, Springer, 2021,
  pp.~153--165.

\bibitem{mnih2015human}
Volodymyr Mnih, Koray Kavukcuoglu, David Silver, Andrei~A Rusu, Joel Veness,
  Marc~G Bellemare, Alex Graves, Martin Riedmiller, Andreas~K Fidjeland, Georg
  Ostrovski, et~al., \emph{Human-level control through deep reinforcement
  learning}, nature \textbf{518} (2015), no.~7540, 529--533.

\bibitem{9536409}
Arunselvan Ramaswamy and Eyke Hüllermeier, \emph{Deep q-learning: Theoretical
  insights from an asymptotic analysis}, IEEE Transactions on Artificial
  Intelligence \textbf{3} (2022), no.~2, 139--151.

\bibitem{scheutzow2013exponential}
Michael Scheutzow, \emph{Exponential growth rate for a singular linear
  stochastic delay differential equation}, Discrete and Continuous Dynamical
  Systems-B \textbf{18} (2013), no.~6, 1683--1696.

\bibitem{sutton2018reinforcement}
Richard~S Sutton and Andrew~G Barto, \emph{Reinforcement learning: An
  introduction}, MIT press, 2018.

\bibitem{szlak2021convergence}
Liran Szlak and Ohad Shamir, \emph{Convergence results for q-learning with
  experience replay}, arXiv preprint arXiv:2112.04213 (2021).

\bibitem{van2016deep}
Hado Van~Hasselt, Arthur Guez, and David Silver, \emph{Deep reinforcement
  learning with double q-learning}, Proceedings of the AAAI conference on
  artificial intelligence, vol.~30, 2016.

\bibitem{wang2020reinforcement}
Jingkang Wang, Yang Liu, and Bo~Li, \emph{Reinforcement learning with perturbed
  rewards}, Proceedings of the AAAI conference on artificial intelligence,
  vol.~34, 2020, pp.~6202--6209.

\bibitem{wang2016dueling}
Ziyu Wang, Tom Schaul, Matteo Hessel, Hado Hasselt, Marc Lanctot, and Nando
  Freitas, \emph{Dueling network architectures for deep reinforcement
  learning}, International conference on machine learning, PMLR, 2016,
  pp.~1995--2003.

\bibitem{watkins1992q}
Christopher~JCH Watkins and Peter Dayan, \emph{Q-learning}, Machine learning
  \textbf{8} (1992), 279--292.

\bibitem{xu2012url}
Cheng-Zhong Xu, Jia Rao, and Xiangping Bu, \emph{Url: A unified reinforcement
  learning approach for autonomic cloud management}, Journal of Parallel and
  Distributed Computing \textbf{72} (2012), no.~2, 95--105.

\bibitem{yin2008q}
G~Yin, CZ~Xu, and LY~Wang, \emph{Q-learning algorithms with random truncation
  bounds and applications to effective parallel computing}, Journal of
  optimization theory and applications \textbf{137} (2008), no.~2, 435--451.

\bibitem{zhang2023convergence}
Shuai Zhang, Hongkang Li, Meng Wang, Miao Liu, Pin-Yu Chen, Songtao Lu, Sijia
  Liu, Keerthiram Murugesan, and Subhajit Chaudhury, \emph{On the convergence
  and sample complexity analysis of deep q-networks with $\epsilon$-greedy
  exploration}, arXiv preprint arXiv:2310.16173 (2023).

\end{thebibliography}

\end{document}